\documentclass{article}

\PassOptionsToPackage{numbers, compress}{natbib}
\usepackage[final]{neurips_2021}

\usepackage[utf8]{inputenc} 
\usepackage[T1]{fontenc}    
\usepackage{hyperref}       
\usepackage{url}            
\usepackage{booktabs}       
\usepackage{amsfonts}       
\usepackage{nicefrac}       
\usepackage{microtype}      
\usepackage{xcolor}         
\usepackage{amsmath}
\usepackage{amssymb}
\usepackage{amsthm}
\usepackage{xcolor}
\usepackage[bbgreekl]{mathbbol}
\usepackage[ruled,vlined]{algorithm2e}
\usepackage{wrapfig}
\usepackage{bm}
\usepackage{graphicx}
\usepackage{float}
\usepackage{subcaption}
\usepackage{comment}
\usepackage{thm-restate}
\usepackage{enumerate}
\usepackage{dblfloatfix} 
\usepackage{hyperref}

\DeclareSymbolFontAlphabet{\mathbb}{AMSb}
\DeclareSymbolFontAlphabet{\mathbbl}{bbold}

\newtheorem{proposition}{Proposition}
\newtheorem{definition}{Definition}

\newtheorem{lemma}{Lemma}

\newcommand{\cp}{\citep}
\newcommand{\ca}{\citeauthor}

\newcommand{\ct}{\citet}
\newcommand{\ctp}[1]{\ca{#1}'s~\cp{#1}}

\newcommand{\R}{\ensuremath{\mathbb{R}}}

\newcommand{\E}{\ensuremath{\mathbb{E}}}

\newcommand{\set}[1]{\ensuremath{\mathcal{#1}}}
\renewcommand{\S}{\set{S}}
\newcommand{\A}{\set{A}}
\newcommand{\V}{\set{V}}
\newcommand{\T}{\set{T}}
\newcommand{\Tpi}{\ensuremath{\T_\pi}}

\newcommand{\M}{\set{M}}

\newcommand{\PP}{\set{P}}

\newcommand{\PS}{\ensuremath{\mathbbl{\Pi}}}
\newcommand{\PSdet}{\ensuremath{\mathbbl{\Pi}_{\mathrm{det}}}}
\newcommand{\FS}{\ensuremath{\mathbbl{{V}}}}
\newcommand{\MS}{\ensuremath{\mathbbl{{M}}}}

\newcommand{\mat}[1]{\ensuremath{\boldsymbol{{#1}}}}

\renewcommand{\v}{\mat{v}}

\newcommand{\defi}{\ensuremath{\equiv}}

\newcommand{\apx}[1]{\ensuremath{\tilde{#1}}}
\newcommand{\pt}{\apx{p}}
\newcommand{\rt}{\apx{r}}

\newcommand{\mt}{\apx{m}}
\newcommand{\vt}{\apx{v}}

\newcommand{\Tpit}{\ensuremath{\apx{\T}_{\pi}}}

\newcommand{\citeproxy}[1]{\textbf{[?]}}

\newcommand{\loss}{\ensuremath{\ell}}

\title{Proper Value Equivalence}

\author{%
Christopher Grimm \\
Computer Science \& Engineering \\
University of Michigan \\
\texttt{crgrimm@umich.edu}
\And
\textbf{Andr\'{e} Barreto, Gregory Farquhar,}\\ \textbf{David Silver, Satinder Singh} \\
DeepMind \\
\hspace{-35pt}\texttt{\{andrebarreto,gregfar,}\\ \hspace{20pt}\texttt{davidsilver,baveja\}@google.com}
}

\begin{document}

\maketitle

\begin{abstract}

One of the main challenges in model-based reinforcement learning (RL) is to decide which aspects of the environment should be modeled. The value-equivalence (VE) principle proposes a simple answer to this question: a model should capture the aspects of the environment that are relevant for value-based planning. Technically, VE distinguishes models based on a set of policies and a set of functions: a model is said to be VE to the environment if the Bellman operators it induces for the policies yield the correct result when applied to the functions. As the number of policies and functions increase, the set of VE models shrinks, eventually collapsing to a single point corresponding to a perfect model. A fundamental question underlying the VE principle is thus how to select the smallest sets of policies and functions that are sufficient for planning. In this paper we take an important step towards answering this question. We start by generalizing the concept of VE to order-$k$ counterparts defined with respect to $k$ applications of the Bellman operator. This leads to a family of VE classes that increase in size as $k \rightarrow \infty$. In the limit, all functions become value functions, and we have a special instantiation of VE which we call proper VE or simply PVE. Unlike VE, the PVE class may contain multiple models even in the limit when all value functions are used. Crucially, all these models are sufficient for planning, meaning that they will yield an optimal policy despite the fact that they may ignore many aspects of the environment. We construct a loss function for learning PVE models and argue that popular algorithms such as MuZero can be understood as minimizing an upper bound for this loss. We leverage this connection to propose a modification to MuZero and show that it can lead to improved performance in practice.
\end{abstract}

\section{Introduction}

It has long been argued that, in order for reinforcement learning (RL) agents to solve truly complex tasks, they must build a model of the environment that allows for counterfactual reasoning~\cp{russel2003artificial}. Since representing the world in all its complexity is a hopeless endeavor, especially under capacity constraints, the agent must be able to ignore aspects of the environment that are irrelevant for its purposes. This is the premise behind the \emph{value equivalence} (VE) principle, which provides a formalism for focusing on the aspects of the environment that are crucial for value-based planning~\cp{grimm2020value}. 

VE distinguishes models based on a set of policies and a set of real-valued scalar functions of state (henceforth, just functions). Roughly, a model is said to be VE to the environment if the Bellman operators it induces for the policies yield the same result as the environment's Bellman operators when applied to the functions. The policies and functions thus become a ``language'' to specify which parts of the environment a model should capture. As the number of policies and functions increase the requirements on the model become more stringent, which is to say that the class of VE models shrinks. In the limit, the VE class collapses to a single point corresponding to a perfect model. Although this result is reassuring, in practice we want to stop short of collapsing---after all, at this point the agent is no longer ignoring irrelevant aspects of the environment.

A fundamental question is thus how to select the smallest sets of policies and functions such that a resulting VE model is sufficient for planning. 
In this paper we take an important additional step in this direction: we show that the VE principle can be formulated with respect to \emph{value functions} only. This result drastically reduces the space of functions that must be considered by VE, as in general only a small fraction of the set of all functions will qualify as value functions in a given environment. Since every policy has an associated value function, this new formulation of VE removes the need for selecting functions, only requiring policies. We name our new formulation \emph{proper value equivalence} (PVE) to emphasize its explicit use of value functions.

PVE has several desirable properties. Unlike with VE, the class of PVE models does not collapse to a singleton in the limit. This means that, even if \emph{all} value functions are used, we generally end up with multiple PVE models---which can be beneficial if some of these are easier to learn or represent than others. Crucially, all of these models are sufficient for planning, meaning that \emph{they will yield an optimal policy despite the fact that they may ignore many aspects of the environment}.

Finally, we make more precise \ctp{grimm2020value} suggestion that the VE principle may help explain the good empirical performance of several modern algorithms~\cp{tamar2016value,silver2017predictron,oh2017value,farquhar2018treeqn,schrittwieser2019mastering}. Specifically, we show that, with mild assumptions, minimizing the loss of the MuZero algorithm  \cp{schrittwieser2020mastering} can be understood as minimizing a PVE error. We then leverage this connection to suggest a modification to MuZero and show a small but significant improvement in the Atari Learning Environment~\citep{bellemare2013arcade}. 

\section{Background}

The agent's interaction with the environment will be modeled as a \emph{Markov decision process} (MDP) $\mathcal{M} \defi \langle \S, \A, r, p, \gamma \rangle$, where $\S$ and $\A$ are the state and action spaces, $r(s,a)$ is the expected reward following taking $a$ from $s$, $p(s' | s, a)$ is the transition kernel and $\gamma \in [0,1)$ is a discount factor~\cp{puterman94markov}. 
A \emph{policy} is a mapping $\pi: \S \mapsto \PP(\A)$, where $\PP(\A)$ is the space of probability distributions over \A; we define $\PS \defi \{\pi \,|\,  \pi: \S \mapsto \PP(\A)\}$ as the set of all possible policies. A policy $\pi$ is \emph{deterministic} if $\pi(a|s) > 0$ for only one action $a$ per state $s$. A policy's \emph{value function} is defined as
\begin{equation} 
\label{eq:v}
v_\pi(s) \defi \E_{\pi} \Big[ \sum_{i=0}^{\infty} \gamma^{i} r(S_{t+i}, A_{t+i}) \,|\, S_{t} = s \Big],
\end{equation}
where $\E_{\pi}[\cdot]$ denotes expectation over the trajectories induced by $\pi$ and the random variables $S_t$ and $A_t$ indicate the state occupied and the action selected by the agent at time step $t$. 

The agent's goal is to find a policy $\pi \in \PS$ that maximizes the value of every state~\cp{sutton2018reinforcement,szepesvari2010algorithms}. Usually, a crucial step to carry out this search is to compute the value function of candidate policies. This process can be cast in terms of the policy's \emph{Bellman operator}:
\begin{equation}
\label{eq:bo}
\Tpi[v](s) \defi \E_{A \sim \pi(\cdot | s), S' \sim  p(\cdot | s, A)} \left[r(s,A)  + \gamma v(S')\right],
\end{equation}
where $v$ is any function in the space $\FS \defi \{f \,|\, f: \S \mapsto \R\}$. It is known that $\lim_{n \rightarrow \infty} \Tpi^n v = v_{\pi}$, that is, starting from any $v \in \FS$, the repeated application of \Tpi\ will eventually converge to $v_{\pi}$.
Since in RL the agent does not know $p$ and $r$, it cannot apply~\eqref{eq:bo} directly. One solution is to learn a \emph{model} $\tilde{m} \defi (\rt,\pt)$ and use it to compute~(\ref{eq:bo}) with $p$ and $r$ replaced by \pt\ and \rt~\cp{sutton2018reinforcement}. We denote the set of all models as $\MS$.

The \emph{value equivalence principle} defines a model as value equivalent (VE) to the environment $m^* \defi (r,p)$ with respect to a set of policies $\Pi$ and a set of functions \V\ if it produces the same Bellman updates as $m^*$ when using $\Pi$ and \V~\cp{grimm2020value}. Classes of such models are expressed as follows:
\begin{equation}
\mathcal{M}(\Pi, \mathcal{V}) \defi \{ \tilde{m} \in \mathcal{M} : \tilde{\Tpi} v = \Tpi v \ \forall \pi \in \Pi, v \in \mathcal{V} \}
\end{equation}
where $\mathcal{M} \subseteq \MS$ is a class of models, $\tilde{\Tpi} $ denotes one application of the Bellman operator induced by model $\tilde{m}$ and policy $\pi$ to function $v$, and $\Tpi$ is environment's Bellman operator for $\pi$.

\citet{grimm2020value} showed that the VE principle can be used to learn models that disregard aspects of the environment which are not related to the task of interest.\footnote{A related approach is taken in value-aware model learning~\citep{farahmand2017value} which minimizes the discrepancy between the Bellman optimality operators induced by the model and the environment.} Classical approaches to model learning do not take the eventual use of the model into account, potentially modeling irrelevant aspects of the environment. Accordingly, \citet{grimm2020value} have shown that, under the same capacity constraints, models learned using VE can outperform their classical counterparts.
\section{Proper value equivalence}
\label{sec:pve}
One can define a spectrum of VE classes corresponding to different numbers of applications of the Bellman operator. We define an order-$k$ VE class as:
\begin{equation}
\label{eq:ve_sequence}
\mathcal{M}^k(\Pi, \mathcal{V}) \defi \{ \tilde{m} \in \mathcal{M} : \tilde{\mathcal{T}}_\pi^k v = \mathcal{T}^k_\pi v \, \, \forall \pi \in \Pi, v \in \mathcal{V} \}
\end{equation}
where $\tilde{\mathcal{T}}_\pi^k v$ denotes $k$ applications of $\tilde{\mathcal{T}}_\pi$ to $v$. Under our generalized definition of VE, \citet{grimm2020value} studied order-one VE classes of the form $\mathcal{M}^{1}(\Pi, \mathcal{V})$. 
They have shown that $\mathcal{M}^1(\mathbbl{\Pi}, \mathbbl{V})$ either contains only the environment or is empty.
This is not generally true for $k > 1$. 
The limiting behavior of order-$k$ value equivalent classes can be described as follows
\begin{restatable}{proposition}{moduloResult}
\label{property:limiting_model_classes}
Let $\mathcal{V}$ be a set of functions such that if $v \in \mathcal{V}$ then $\mathcal{T}_{\pi} v \in \mathcal{V}$ for all $\pi \in \Pi$.
Then, for $k, K \in \mathbb{Z}^+$ such that $k$ divides $K$, it follows that:
\begin{enumerate}[(i)]
    \item For any $\M \subseteq \MS$ and any $\Pi \subseteq \mathbbl{\Pi}$, we have that $\mathcal{M}^k(\Pi, \V) \subseteq \mathcal{M}^K(\Pi, \V)$. 
    \item If $\Pi$ is non-empty and \V\ contains at least one constant function, then there exist environments such that $\MS^k(\Pi, \V) \subset \MS^K(\Pi, \V)$.
\end{enumerate}
\end{restatable} 
We defer all proofs of theoretical results to Appendix~\ref{sec:proofs}.
Based on Proposition~\ref{property:limiting_model_classes} we can relate different VE model classes according to the greatest common divisor of their respective orders; specifically, two classes $\M^k(\Pi, \V)$ and $\M^K(\Pi, \V)$ will intersect at $\M^{\mathrm{gcd}(k,K)}(\Pi, \V)$ (Figure~\ref{fig:set_inclusion}).
\begin{wrapfigure}{R}{0.4\textwidth}
\vspace{-20pt}
\centering 
\includegraphics[width=0.38\textwidth]{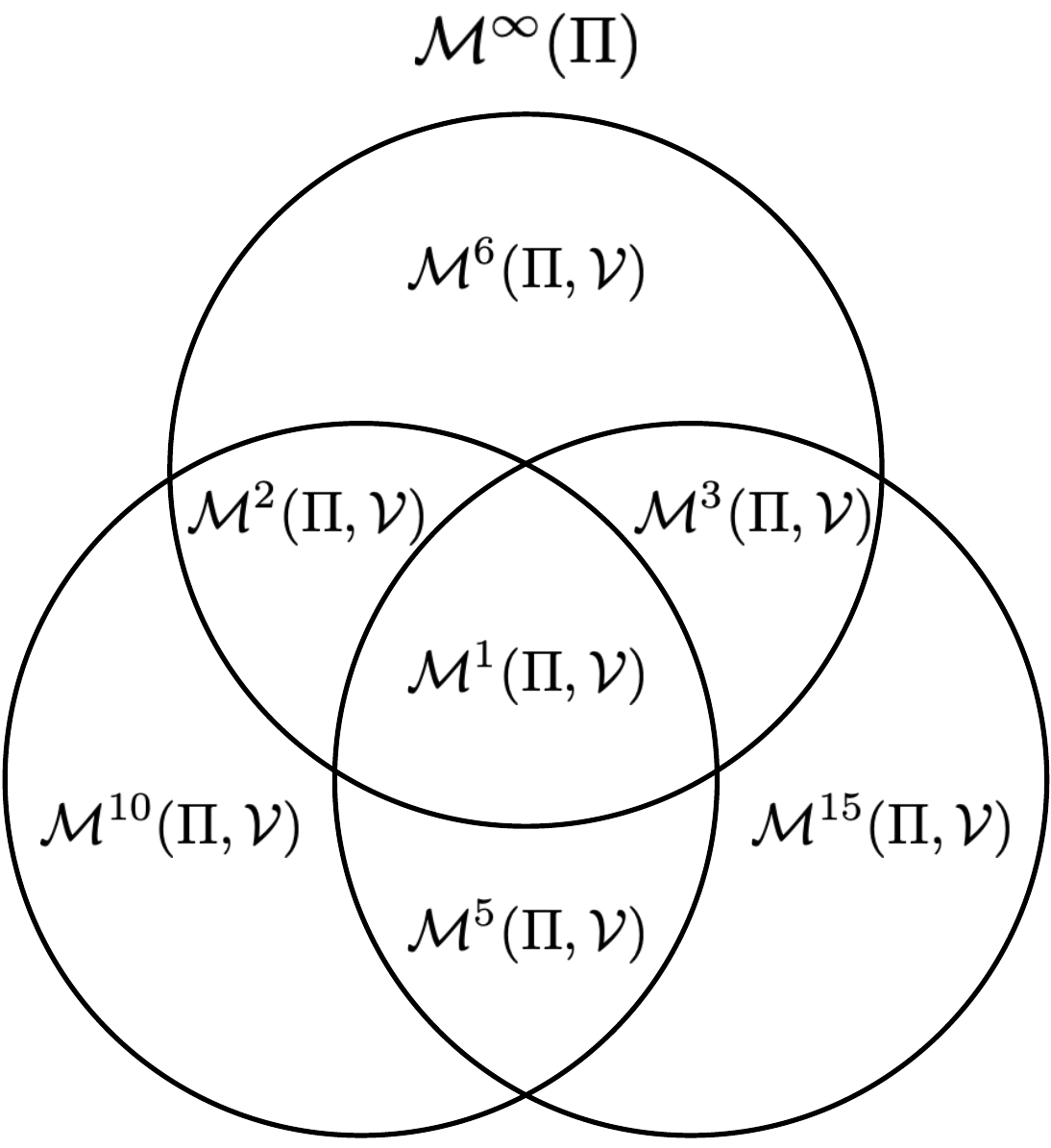}
\caption{Topology of the space of order-$k$ VE classes. Given a set of policies $\Pi$, a set of functions $\V$ closed under Bellman updates, and $k, K \in \mathbb{Z}^+$ such that $k$ divides $K$, we have that $\M^k(\Pi, \V)~\subseteq~\M^K(\Pi, \V)$.}
\label{fig:set_inclusion}
\vspace{-30pt}
\end{wrapfigure}
Proposition~\ref{property:limiting_model_classes} also implies that, in contrast to order-one VE classes, higher order VE classes potentially include multiple models, even if VE is defined with respect to all policies $\mathbbl{\Pi}$ and all functions $\mathbbl{V}$. In addition, the size of a VE class cannot decrease as we increase its order from $k$ to a multiple of $k$ (and in some cases it will strictly increase). This invites the question of what happens in the limit as we keep increasing the VE order. To answer this question, we introduce a crucial concept for this paper:
\begin{definition}
(Proper value equivalence). Given a set of policies $\Pi \subseteq \mathbbl{\Pi}$, let 
\begin{equation}
\mathcal{M}^\infty(\Pi) = \lim_{k \to \infty} \mathcal{M}^k(\Pi, \mathbbl{V}) = \{ \tilde{m} \in \mathcal{M} : \tilde{v}_\pi = v_\pi \ \forall \pi \in \Pi \},
\label{eq:pve-def}
\end{equation}
where $\tilde{v}_\pi$ and $v_\pi$ are the value functions of $\pi$ induced by model $\tilde{m}$ and the environment. We say that each $\tilde{m} \in \mathcal{M}^\infty(\Pi)$ is a \textbf{proper value equivalent} model to the environment with respect to $\Pi$. 
\end{definition}
Because the process of repeatedly applying a policy's Bellman operator to a function converges to the same fixed point regardless of the function, in an order-$\infty$ VE class the set $\Pi$ uniquely determines the set \V. This reduces the problem of defining $\Pi$ and $\V$ to defining the former only. Also, since all functions in an order-$\infty$ VE are \emph{value} functions, we call it \emph{proper VE} or PVE.

It is easy to show that Proposition~\ref{property:limiting_model_classes} is valid for any $k \in \mathbb{Z}^+$ when $K = \infty$ (Corollary~\ref{coro:FPVE_vs_VE_limiting} in Appendix~\ref{sec:proofs}).
Thus, in some sense, $\M^{\infty}$ is the ``biggest'' VE class. It is also possible to define this special VE class in terms of any other:
\begin{restatable}{proposition}{FPVEDecomp}
\label{prop:fpve_decomp}
For any $\Pi \subseteq \mathbbl{\Pi}$ and any $k \in \mathbb{Z}^+$ it follows that
\begin{equation}
\mathcal{M}^\infty(\Pi) = \bigcap_{\pi \in \Pi} \M^k(\{ \pi \}, \{ v_\pi \}),
\label{eq:fpve_decomp}
\end{equation}
\end{restatable}
where $v_\pi$ is the value of policy $\pi$ in the environment.

We thus have two equivalent ways to describe the class of models which are PVE with respect to a set of policies $\Pi$.
The first, given in (\ref{eq:pve-def}), is the order-$\infty$ limit of value equivalence with respect to $\Pi$ and the set of all functions $\mathbbl{V}$.
The second, given in (\ref{eq:fpve_decomp}), is the intersection of the classes of models that are order-$k$ VE with respect to the singleton policies in $\Pi$ and their respective value functions.
This latter form is valid for \emph{any} $k$, and will underpin our practical algorithmic instantiations of PVE.

Setting $k=1$ in Proposition~\ref{prop:fpve_decomp} we see that PVE can be written in terms of order-one VE. This means that $\M^{\infty}$ inherits many of the topological properties of $\M^{1}$ shown by \citet{grimm2020value}. Specifically, we know that $\mathcal{M}'^\infty(\Pi) \subseteq \mathcal{M}^\infty(\Pi)$ if $\mathcal{M}' \subseteq \mathcal{M}$ and also that $\mathcal{M}^\infty(\Pi') \subseteq \mathcal{M}^\infty(\Pi)$ when $\Pi \subseteq \Pi'$ (these directly follow from \ctp{grimm2020value} Properties 1 and 3 respectively).

Proposition~\ref{prop:fpve_decomp} also sheds further light into the relation between PVE and order-$k$ VE more generally. Let $\Pi$ be a set of policies and $\V_{\pi}$ their value functions. Then, for any $k \in \mathbb{Z}^+$, we have that
\begin{equation}
\label{eq:pve_ve}
\begin{aligned}
\M^k(\Pi,\V_{\pi}) &= \bigcap_{\pi \in \Pi}  \bigcap_{v \in \V_\pi} \M^k(\{ \pi \}, \{ v \}) \subseteq \bigcap_{\pi \in \Pi} \M^k(\{ \pi \}, \{ v_\pi \}) = \M^\infty(\Pi),
\end{aligned}
\end{equation}
which is another way to say that $\M^{\infty}$ is, in some sense, the largest among all the VE classes. The reason why the size of VE classes is important is that it directly reflects the main motivation behind the VE principle. VE's premise is that models should be constructed taking into account their eventual use: if some aspects of the environment are irrelevant for value-based planning, it should not matter whether a model captures them or not. This means that all models that only differ with respect to these irrelevant aspects but are otherwise correct qualify as valid VE solutions. A larger VE class generally means that more irrelevant aspects of the environment are being ignored by the agent. We now make this intuition more concrete by showing how irrelevant aspects of the environment that are eventually captured by order-one VE are always ignored by PVE:

\begin{restatable}{proposition}{irrelevantState}
\label{prop:FPVE_strictly_bigger}
Let $\Pi \subseteq \PS$. If the environment state can be factored as $\mathcal{S} = \mathcal{X} \times \mathcal{Y}$ where $|\mathcal{Y}| > 1$ and $v_\pi(s) = v_\pi((x, y)) = v_\pi(x)$ for all $\pi \in \Pi$, then $\MS^1(\Pi, \mathbbl{V}) \subset \MS^\infty(\Pi)$.
\end{restatable}

Note that the subset relation appearing in Proposition~\ref{prop:FPVE_strictly_bigger} is strict.
We can think of the variable `$y$' appearing in Proposition~\ref{prop:FPVE_strictly_bigger} as superfluous features that do not influence the RL task, like the background of an image or any other sensory data that is irrelevant to the agent's goal. A model is free to assign arbitrary dynamics to such irrelevant aspects of the state without affecting planning performance. Since order-one VE eventually pins down a model that describes everything about the environment, one would expect the size of $\M^{\infty}$ relative to $\M^1$ to increase as more superfluous features are added. Indeed, in our proof of Proposition~\ref{prop:FPVE_strictly_bigger} we construct a set of models in $\mathcal{M}^\infty(\PS)$ which are in one-to-one correspondence with $\mathcal{Y}$, confirming this intuition (see Appendix~\ref{sec:proofs}).

\subsection*{Proper value equivalence yields models that are sufficient for optimal planning}

In general PVE does not collapse to a single model even in the limit of $\Pi = \PS$. At first this may cause the impression that one is left with the extra burden of selecting one among the PVE models. However, it can be shown that no such choice needs to be made: 
\begin{restatable}{proposition}{FPVEOptimal}
\label{prop:fpve_limiting_model_class_optimal2}
An optimal policy for any $\tilde{m} \in \M^\infty(\PS)$ is also an optimal policy in the environment. 
\end{restatable}

According to Proposition~\ref{prop:fpve_limiting_model_class_optimal} any model $\tilde{m} \in \M^\infty(\PS)$ used for planning will yield an optimal policy for the environment. In fact, in the spirit of ignoring as many aspects of the environment as possible, we can define an even larger PVE class by focusing on deterministic policies only:

\begin{restatable}{corollary}{FPVEDetOptimal}
\label{prop:fpve_limiting_model_class_optimal}
Let $\mathbbl{\Pi}_\text{det}$ be the set of all deterministic policies. An optimal policy for any $\tilde{m} \in \mathcal{M}^\infty(\mathbbl{\Pi}_\text{det})$ is also optimal in the environment. 
\end{restatable}

Given that both $\M^{\infty}(\PS)$ and $\M^{\infty}(\PS_{\text{det}})$ are sufficient for optimal planning, one may wonder if these classes are in fact the same. The following result states that the class of PVE models with respect to deterministic policies can be strictly larger than its counterpart defined with respect to all policies:

\begin{restatable}{proposition}{FPVEDetBigger}
\label{prop:fpve_det_bigger}
There exist environments and model classes for which $\mathcal{M}^\infty(\mathbbl{\Pi}) \subset \mathcal{M}^\infty(\PSdet)$.
\end{restatable}

\begin{figure}
\centering 
\includegraphics[width=0.85\textwidth]{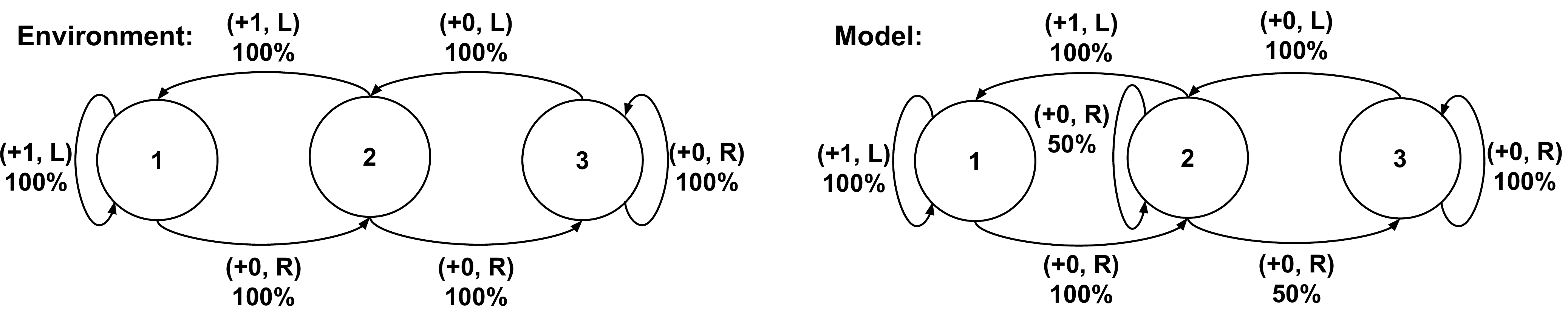}
\caption{An environment / model pair with the same values for all deterministic policies but not all stochastic policies. The environment has three states and two actions: $\A = \{\mathrm{L}, \mathrm{R}\}$. The percentages in the figure indicate the probability of a given transition and the corresponding tuples $(r, a)$ indicate the reward associated with a given action. A deterministic policy cannot dither between $s_1$ and $s_3$ but a stochastic policy can. Note that the dynamics between the pair differs when taking action R from $s_2$. This difference will affect the dithering behavior of such a stochastic policy in a way that results in different model and environment values.}
\label{fig:fpve_example}
\end{figure}
Figure~\ref{fig:fpve_example} illustrates Proposition~\ref{prop:fpve_det_bigger} with an example of environment and a model $\tilde{m}$ such that $\tilde{m} \in \mathcal{M}^\infty(\mathbbl{\Pi}_\text{det})$ but 
$\tilde{m} \notin \mathcal{M}^\infty(\mathbbl{\Pi})$. 

To conclude our discussion on models that are sufficient for optimal planning, we argue that, in the absence of additional information about the environment or the agent, $\mathcal{M}^\infty(\mathbbl{\Pi}_\text{det})$ is in fact the largest possible VE class that is guaranteed to yield optimal performance.
To see why this is so, suppose we remove a single deterministic policy from $\PSdet$ and pick an arbitrary model $\tilde{m} \in \M^\infty(\PSdet - \{\pi\})$. Let $\tilde{v}_\pi$ be the value function of $\pi$ according to the model $\tilde{m}$. Because $\pi$ is not included in the set of policies used to enforce PVE, $\tilde{v}_\pi$ may not coincide with $v_\pi$, the actual value function of $\pi$ according to the environment. Now, if $\pi$ happens to be the only optimal policy in the environment and $\tilde{v}_\pi$ is not the optimal value function of $\tilde{m}$, the policy returned by this model will clearly be sub-optimal.

\section{Learning a proper value-equivalent model}
\label{sec:learning_pve}

Having established that we want to find a model in $\M^\infty(\PSdet)$, we now turn our attention to how this can be done in practice. Following \ct{grimm2020value}, given a finite set of policies $\Pi$ and a finite set of functions \V, we cast the search for a model $\mt \in \M^k(\Pi, \V)$ as the minimization of deviations from~(\ref{eq:ve_sequence}):
\begin{equation}
\label{eq:k_ve_loss}
\loss_{\Pi, \V}^k(m^*,\mt) \defi \sum_{\pi \in \Pi} \sum_{v \in \V} \| \Tpi^k v - \Tpit^k v||,
\end{equation}
where $\Tpit $ are Bellman operators induced by \mt\ and $\lVert\cdot\rVert$ is a norm.\footnote{We can also impose VE with infinite sets of functions and policies by replacing the respective sums with integrals; in this case one may consider taking a supremum over VE terms to avoid situations where VE is not necessarily satisfied on measure 0 sets.} Note that setting $k = \infty$ in~(\ref{eq:k_ve_loss}) yields a loss that requires computing $\mt$'s value function---which is impractical to do if $\mt$ is being repeatedly updated. Thankfully, by leveraging the connection between order-$k$ VE and PVE given in Proposition~\ref{prop:fpve_decomp}, we can derive a practical PVE loss:
\begin{equation}
\label{eq:pve_loss}
\loss_{\Pi, \infty}^{k}(m^*,\mt) \defi 
\sum_{\pi \in \Pi} \| \Tpi^k v_\pi - \Tpit^k v_\pi||
= \sum_{\pi \in \Pi} \| v_\pi - \Tpit^k v_\pi||.
\end{equation}
Interestingly, given a set of policies $\Pi$, minimizing~(\ref{eq:pve_loss}) for any $k$ will result in a model $\mt \in \M^\infty(\Pi)$ ({\sl cf.} Proposition~\ref{prop:fpve_decomp}). As we will discuss shortly, this property can be exploited to generate multiple loss functions that provide a richer learning signal in practical scenarios.

Contrasting loss functions~(\ref{eq:k_ve_loss}) and~(\ref{eq:pve_loss}) we observe an important fact: unlike with other order-$k$ VE classes, PVE requires actual value functions to be enforced in practice. Since value functions require data and compute to be obtained, it is reasonable to ask whether the benefits of PVE justify the associated additional burden. Concretely, one may ask whether the sample transitions and computational effort spent in computing the value functions to be used with PVE would not be better invested in enforcing other forms of VE over arbitrary functions that can be readily obtained.

We argue that in many cases one does not have to choose between order-$k$ VE and PVE. Value-based RL algorithms usually compute value functions iteratively, generating a sequence of functions $v_1, v_2, ...$ which will eventually converge to $\vt_\pi$ for some $\pi$. A model-based algorithm that computes $\vt_\pi$ in this way has to somehow interleave this process with the refinement of the model $\mt$. When it comes to VE, one extreme solution is to only use the final approximation $\vt_\pi \approx v_\pi$ in an attempt to enforce PVE through~(\ref{eq:pve_loss}). It turns out that, as long as the sequence $v_1, v_2, ...$ is approaching $v_\pi$, one can use \emph{all} the functions in the sequence to enforce PVE with respect to $\pi$. Our argument is based on the following result:
\begin{restatable}{proposition}{PVEBound}
\label{prop:pve_bound}
For any $\pi \in \PS$, $v \in \FS$ and $k, n \in \mathbb{Z}^+$, we have that
\begin{equation}
\label{eq:pve_bound}
\lVert v_\pi - \Tpit^k v_\pi \rVert_\infty \le (\gamma^k + \gamma^n) \underbrace{\lVert v_\pi - v \rVert_\infty}_{\epsilon_v} + \underbrace{\lVert \Tpi^n v  - \Tpit^k v \rVert_\infty}_{\epsilon_{ve}}.
\end{equation}
\end{restatable}
Note that the left-hand side of~(\ref{eq:pve_bound}) corresponds to one of the terms of the PVE loss~(\ref{eq:pve_loss}) associated with a given $\pi$. This means that, instead of minimizing this quantity directly, one can minimize the upper-bound on the right-hand side of~(\ref{eq:pve_bound}). The first term in this upper bound, $\epsilon_v$, is the conventional value-function approximation error that most value-based methods aim to minimize (either directly or indirectly). The second term, $\epsilon_{ve}$, is similar to the terms appearing in the order-$k$ VE loss~(\ref{eq:k_ve_loss}), except that here the number of applications of $\Tpi$ and of its approximation $\Tpit$ do not have to coincide. 

All the quantities appearing in $\epsilon_{ve}$ are readily available or can be easily approximated using sample transitions~\cp{grimm2020value}. Thus, $\epsilon_{ve}$ can be used to refine the model \mt\ using functions $v$ that are not necessarily value functions. As $v \rightarrow v_\pi$, two things happen. First, $\epsilon_{ve}$ approaches one of the terms of the PVE loss~(\ref{eq:pve_loss}) associated with policy $\pi$. Second, $\epsilon_v$ vanishes. Interestingly, the importance of $\epsilon_v$ also decreases with $n$ and $k$, the number of times $\Tpi$ and $\Tpit$ are applied in $\epsilon_{ve}$, respectively. This makes sense: since $\Tpi^n v \rightarrow v_\pi$ as $n \rightarrow \infty$ and, by definition, VE approaches PVE as $k \rightarrow \infty$, we have that $\epsilon_{ve}$ approaches the left-hand side of~(\ref{eq:pve_bound}) as both $n$ and $k$ grow.

\subsection*{An extended example: MuZero through the lens of value equivalence}
\label{sec:muzero_connection}

\ct{grimm2020value} suggested that the VE principle might help to explain the empirical success of recent RL algorithms like Value Iteration Networks, the Predictron, Value Prediction Networks, TreeQN, and MuZero~\cp{tamar2016value,silver2017predictron,oh2017value,farquhar2018treeqn,schrittwieser2020mastering}. In this section we investigate this hypothesis further and describe a possible way to interpret one of these algorithms, MuZero, through the lens of VE. We acknowledge that the derivation that follows abstracts away many details of MuZero and involves a few approximations of its mechanics, but we believe it captures and explains the algorithm's essence.

MuZero is a model-based RL algorithm that achieved state-of-the-art performance across both board games, such as Chess and Go, and Atari 2600 games~\cp{schrittwieser2020mastering}.
The model $\tilde{m}$ in MuZero is trained on sequences of states, actions and rewards resulting from executing a ``behavior policy'' in the environment: $s_{t:t+n+K}, a_{t:t+n+K}, r_{t:t+n+K}$ where $n$ and $K$ are hyperparameters of the agent which will be explained shortly. The agent produces an ``agent state'' $z^0_t$ from $s_t$ and subsequently generates $z^{1:K}_t$ by using its model to predict the next $K$ agent states following actions $a_{t:t+K}$. The agent also maintains reward and value function estimates as a function of agent states, which we denote $\tilde{r}(z)$ and $v(z)$ respectively. A variant\footnote{In reality MuZero uses a categorical representation for its value and reward functions and minimizes them using a cross-entropy objective. We argue that this choice is not essential to its underlying ideas and use scalar representations with a squared loss to simplify our analysis.} of MuZero's per-state model loss can thus be expressed as:
\begin{equation}
\label{eq:model_loss}
\ell^{\mu}(s_t) =  \sum_{k=0}^K (V_{t+k} - v(z^k_t))^2  + ( r_{t+k} - \tilde{r}(z_t^k)) )^2
\end{equation}
where $V_{t+k} = r_{t+k} + \cdots + \gamma^{n-1} r_{t+k+n-1} + \gamma^n v^{targ}(z^0_{t+k+n})$. The term $v^{targ}$ is a value target produced by Monte-Carlo tree search (MCTS, \citep{coulom2006efficient}). Because the behavior policy is itself computed via MCTS, we have that $v^{targ} \approx v$; for simplicity we will assume that $v^{targ} = v$ and only use $v$.

In what follows we show, subject to a modest smoothness assumption, that minimizing MuZero's loss with respect to its behavior policy, $\pi$, also minimizes a corresponding PVE loss. Put precisely:
\begin{equation}
C \cdot \mathbb{E}_{d_\pi}[ \ell^\mu(S_t) ] \geq \big( \ell^K_{\{ \pi \}, \infty}(m^*, \tilde{m}) \big)^2
\end{equation}
for some $C > 0$, where $d_\pi$ is a stationary distribution. We proceed by combining two derivations: a lower-bound on $\mathbb{E}_{d_\pi}[\ell^\mu(S_t)]$ in \eqref{eq:loss_side_bound}, and an upper-bound on $(\ell^K_{\{ \pi \}, \infty}(m^*, \tilde{m}))^2$ in \eqref{eq:square_jensens}. 

As a preliminary step we note that $\ell^\mu(s_t)$ and $\ell^K_{\{\pi\}, \infty}(m^*, \tilde{m})$ are expressed in terms of samples and expectations respectively. We note the following connection between these quantities: 
\begin{equation}
\label{eq:expectations}
\begin{aligned}
& \mathbb{E}[r_{t+k} | s_t] = \mathcal{P}^k_\pi[r_\pi](s_t),
& \mathbb{E}[V_{t+k} | s_t] = \mathcal{P}^k_\pi \mathcal{T}^n_\pi [v_\pi](s_t), \\
& \quad \mathbb{E}[\tilde{r}(z^k_t) | s_t] = \tilde{\mathcal{P}}^k_\pi [\tilde{r}_\pi](s_t),
&\mathbb{E}[v(z^k_t) | s_t] = \tilde{\mathcal{P}}^k_\pi[v_\pi](s_t),
\end{aligned}
\end{equation}
where $\mathcal{P}^k_\pi$ is the $k$-step environment transition operator under policy $\pi$: $\mathcal{P}^k_\pi[x](s_t) = \mathbb{E}[x(S_{t+k}) | s_t, m^*, \pi]$, $r_\pi(s) = \mathbb{E}_{A \sim \pi}[r(s, A)]$ and $\tilde{\mathcal{P}}^k_\pi$ and $\tilde{r}_\pi$ are the corresponding quantities using the model instead of the environment. The above expectations are taken with respect to the environment or model and $\pi$ as appropriate.
We now derive our lower-bounds on $\mathbb{E}_{d_\pi}[\ell^\mu(S_t)]$:
\begin{equation}
\label{eq:jensen}
\begin{aligned}
\mathbb{E}_{d_\pi}[\ell^{\mu}(S_t)] &= \mathbb{E}_{d_\pi}\Big[\sum_{k=0}^K \mathbb{E}[\,( V_{t+k} - v(z_t^k))^2 \mid S_t\,] + \sum_{k=0}^K \mathbb{E}[\,  (r_{t+k} -\tilde{r}(z_t^k))^2 \mid S_t\,]\Big] \\
& \geq 
    \mathbb{E}_{d_\pi}\Big[\sum_{k=0}^K ( \mathbb{E}[\,V_{t+k} \mid S_t\,] - 
    \mathbb{E}[\,v(z^k_t) \mid S_t\,])^2 + \sum_{k=0}^K  ( \mathbb{E}[\,r_{t+k} \mid S_t\,] - \mathbb{E}[\,\tilde{r}(z^k_t) \mid S_t\,] )^2\Big] \\
& =
    \sum_{k=0}^K \mathbb{E}_{d_\pi}\Big[( \mathcal{P}^k_\pi \mathcal{T}^n_\pi v(S_t) - \tilde{\mathcal{P}}^k_\pi v(S_t) )^2\Big] + \sum_{k=0}^K \mathbb{E}_{d_\pi}\Big[( \mathcal{P}^k_\pi r_\pi(S_t) - \tilde{\mathcal{P}}^k_\pi \tilde{r}_\pi(S_t))^2\Big]
\end{aligned}
\end{equation}
where we apply the tower-property, Jensen's inequality and the identities in~\eqref{eq:expectations}. We write the expression using norms and drop all terms except $k \in \{0, K\}$ in the first sum to obtain:
\begin{equation}
\label{eq:loss_side_bound}
\mathbb{E}_{d_\pi}[\ell^\mu(S_t)] \geq \| \mathcal{T}^n_\pi v - v \|^2_{d_\pi} + \| \mathcal{P}^K_\pi \mathcal{T}^n_\pi v - \tilde{\mathcal{P}}^K_\pi v \|^2_{d_\pi} + \sum_{k=0}^K \| \mathcal{P}^k_\pi r_\pi - \tilde{\mathcal{P}}^k_\pi \tilde{r}_\pi \|^2_{d_\pi}
\end{equation}
recalling that $\| x - y \|_{d_\pi}^2 = \mathbb{E}_{d_\pi}[(x(S_t) - y(S_t))^2]$.  
To derive an upper-bound for $(\ell^K_{\{ \pi \}, \infty}(m^*, \tilde{m}))^2$ we assume that the error in value estimation is smooth in the sense that there is some $g > 0$ (independent of $v$) such that  $\| v - v_\pi \|_\infty < g \cdot \| v - v_\pi \|_{d_\pi}$. We can then use a modified version of \eqref{eq:pve_bound} for the $d_\pi$-weighted $\ell_2$-norm (see Appendix~\ref{sec:proofs}), plugging in $n + K$ and $K$:
\begin{equation}
\label{eq:upper_bound_of_upper_bound}
\begin{aligned}
\| v_\pi - \tilde{\mathcal{T}}^K_\pi v_\pi \|_{d_\pi} 
&\leq (g\gamma^K + \gamma^{n+K}) \| v_\pi - v \|_{d_\pi} + \| \mathcal{T}^{K+n}_\pi v - \tilde{\mathcal{T}}^K_\pi v \|_{d_\pi}\\
&\leq \gamma^K(g + \gamma^n) \| v_\pi - v \|_{d_\pi} + \| \mathcal{P}^K_\pi \mathcal{T}^n_\pi v - \tilde{\mathcal{P}}^K_\pi v \|_{d_\pi} + \sum_{k=0}^K \| \mathcal{P}^k_\pi r_\pi - \tilde{\mathcal{P}}_\pi^k \tilde{r}_\pi \|_{d_\pi} \\
&\leq \gamma^K \frac{(g + \gamma^n)}{(1 - \gamma^n)} \| \mathcal{T}^n_\pi v - v \|_{d_\pi} + \| \mathcal{P}^K_\pi \mathcal{T}^n_\pi v - \tilde{\mathcal{P}}^K_\pi v \|_{d_\pi} + \sum_{k=0}^K \| \mathcal{P}^k_\pi r_\pi - \tilde{\mathcal{P}}_\pi^k \tilde{r}_\pi \|_{d_\pi},
\end{aligned}
\end{equation}
from here we can square both sides and apply Jensen's inequality,
\begin{equation}
\label{eq:square_jensens}
\begin{aligned}
\| v_\pi - \tilde{\mathcal{T}}_\pi^K v_\pi \|^2_{d_\pi} \leq ab \| \mathcal{T}^\pi_n v - v \|^2_{d_\pi} + b \| \mathcal{P}^K_\pi \mathcal{T}^n_\pi v - \tilde{\mathcal{P}}^K_\pi v \|^2_{d_\pi} + b \sum_{k=0}^K \| \mathcal{P}^k_\pi r_\pi - \tilde{\mathcal{P}}_\pi^k  \tilde{r}_\pi \|^2_{d_\pi},
\end{aligned}
\end{equation}
where $a = \gamma^K (g + \gamma^n) (1 - \gamma^n)^{-1}$ and $b = a + K + 2$.
Combining \eqref{eq:square_jensens} and \eqref{eq:loss_side_bound} we obtain:
\begin{equation}
\label{eq:conclusion}
ab \cdot \mathbb{E}_{d_\pi}[\ell^{\mu}(S_t)] \geq \| v_\pi - \tilde{\mathcal{T}}_\pi^K v_\pi \|_{d_\pi}^2 = \big(\ell^K_{\{ \pi \}, \infty}(m^*, \tilde{m})\big)^2,
\end{equation}
thus minimizing MuZero's loss minimizes a squared PVE loss with respect to a single policy. 
\section{Experiments}
\label{sec:experiments}
\begin{figure}[t!]
\centering
\small\addtolength{\tabcolsep}{-6pt}
\begin{subfigure}[l]{0.99\textwidth}
\begin{tabular}[c]{cccc}
    \begin{subfigure}[c]{0.19\textwidth}
      \includegraphics[width=\textwidth]{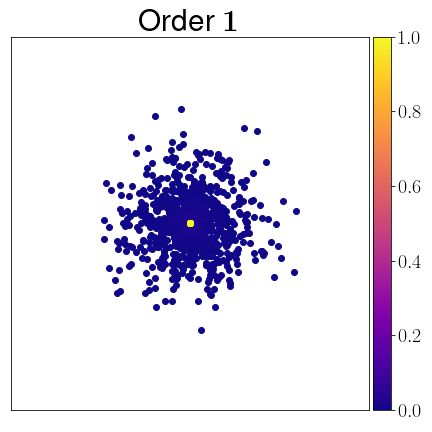}
    \end{subfigure}&
    \begin{subfigure}[c]{0.19\textwidth}
      \includegraphics[width=\textwidth]{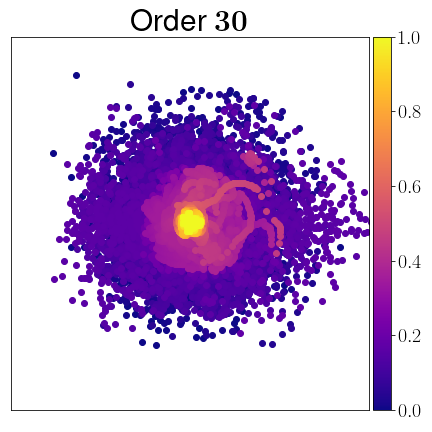}
    \end{subfigure}&
    \begin{subfigure}[c]{0.19\textwidth}
      \includegraphics[width=\textwidth]{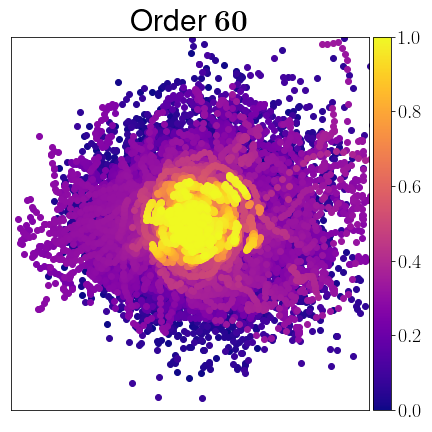}
    \end{subfigure}&
    \begin{subfigure}[c]{0.19\textwidth}
      \includegraphics[width=\textwidth]{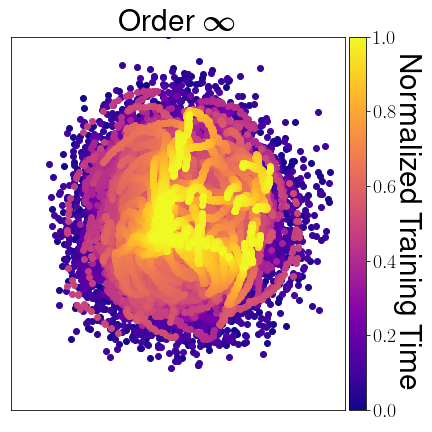}
    \end{subfigure}\\
    \begin{subfigure}[c]{0.19\textwidth}
      \includegraphics[width=\textwidth]{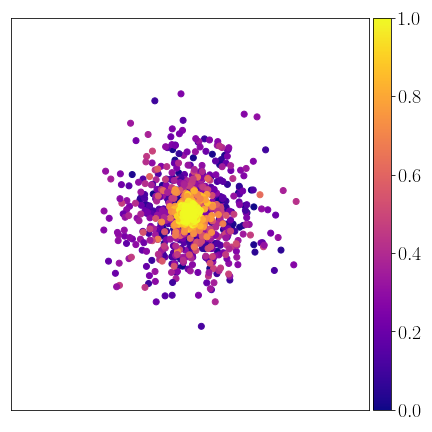}
    \end{subfigure}&
    \begin{subfigure}[c]{0.19\textwidth}
      \includegraphics[width=\textwidth]{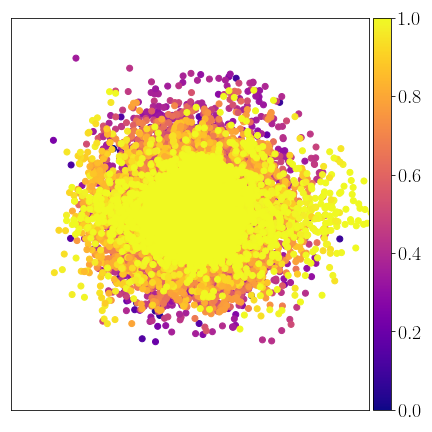}
    \end{subfigure}&
    \begin{subfigure}[c]{0.19\textwidth}
      \includegraphics[width=\textwidth]{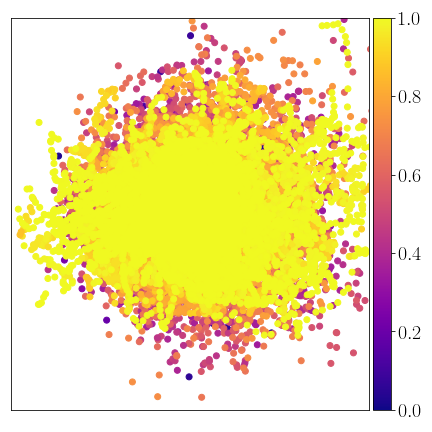}
    \end{subfigure}&
    \begin{subfigure}[c]{0.19\textwidth}
      \includegraphics[width=\textwidth]{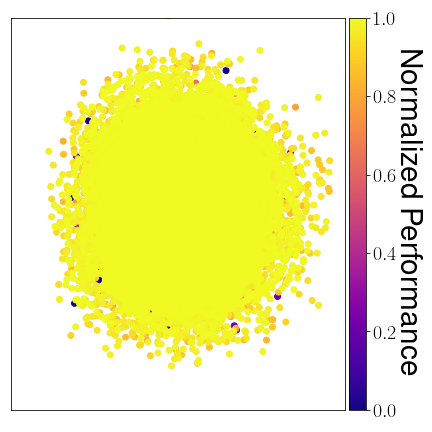}
    \end{subfigure}
  \end{tabular}
\end{subfigure}%
\begin{subfigure}[c]{0.25\textwidth}
\hspace{-95pt}
  \includegraphics[width=\textwidth]{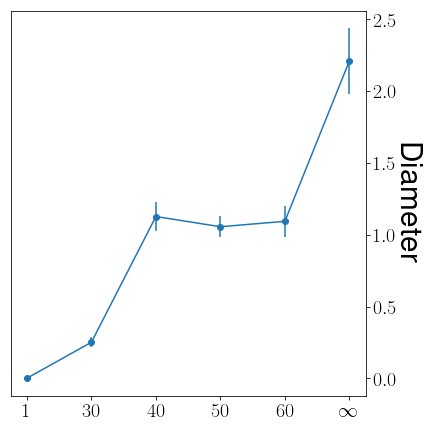}
\end{subfigure}
\caption{All scatter plots are generated by tracking the training progress over 500,000 iterations of models with different order-$k$ VE objectives. In each plot 120 models were tracked; at every 1000 timesteps the full set of models is converted into vector form and projected onto their first two principal components before being plotted (details in the appendix). \textbf{Top row:} points are colored according to their progress through training. \textbf{Bottom row:} points are colored according to the average value of the associated  model's optimal policy on the environment. \textbf{Rightmost plot:} line-plot of the diameters of these scatter plots against the model-class order.}
\label{fig:scatter_plots}
\end{figure} 

We first provide results from tabular experiments on a stochastic version of the Four Rooms domain which serve to corroborate our theoretical claims. Then, we present results from experiments across the full Atari 57 benchmark~\citep{bellemare2013arcade} showcasing that the insights from studying PVE and its relationship to MuZero can provide a benefit in practice at scale. See Appendix~\ref{sec:experimental_details} for a full account of our experimental procedure. 

In Section~\ref{sec:pve} we described the topological relationships between order-$k$ and PVE classes. 
This is summarized by Proposition~\ref{property:limiting_model_classes}, which shows that, for appropriately defined \V\ and $\Pi$, $\M^k \subseteq \M^K$ if $K$ is a multiple of $k$.
We illustrate this property empirically by randomly initializing a set of models and then using~\eqref{eq:k_ve_loss} (or~\eqref{eq:pve_loss} for the limiting case of $k=\infty$) to iteratively update them towards  $\M^{k}(\PS, \FS)$, with $k \in \{1, 30, 40, 50, 60, \infty \}$. We take the vectors representing these models and project them onto their first two principal components in order to visualise their paths through learning.  The results are shown on the top row of Figure~\ref{fig:scatter_plots}.
In accordance with the theory, we see that the space of converged models, represented with the brightest yellow regions, grows with $k$.
This trend is summarised in the rightmost plot, which shows the diameter of the scatter plots for each $k$.
In the bottom row of Figure~\ref{fig:scatter_plots} we use color to show the value that the optimal policy of each model achieves in the true environment.
As predicted by our theory, the space of models that are sufficient for optimal planning also grows with $k$.


\begin{figure}[b!]
\centering
\begin{tabular}[c]{cc}
\begin{subfigure}[c]{0.35\textwidth}
\hspace{-10pt}
\begin{subfigure}[c]{\textwidth}
    \includegraphics[width=\textwidth]{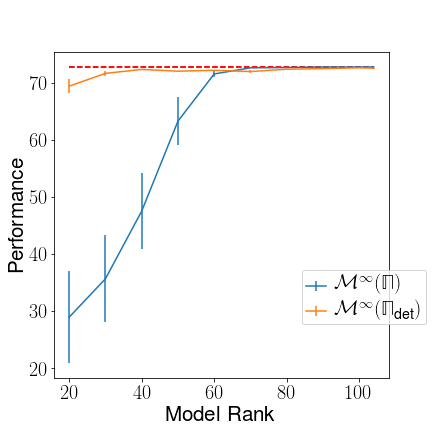}
\end{subfigure}
\caption{}
\end{subfigure}&
\hspace{-25pt}
\begin{subfigure}[c]{0.67\textwidth}
    \begin{subfigure}[c]{0.5\textwidth}
    \includegraphics[width=\textwidth]{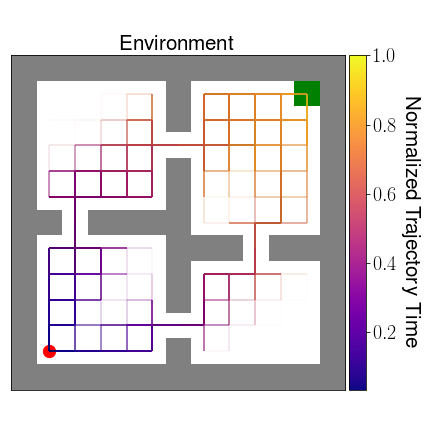}
    \end{subfigure}%
    \begin{subfigure}[c]{0.5\textwidth}
    \includegraphics[width=\textwidth]{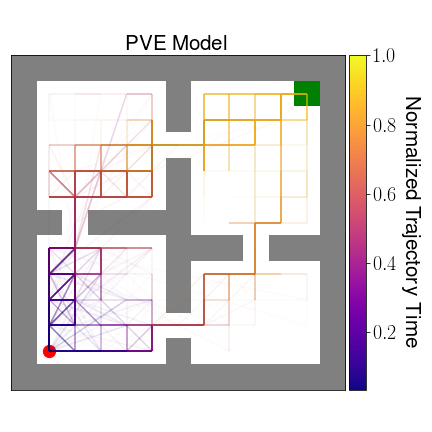}
    \end{subfigure}
    \caption{} 
\end{subfigure}
\end{tabular}
\caption{\textbf{(a)} Comparison of the performance of optimal policies obtained from capacity constrained models trained to be in $\M^\infty(\PSdet)$ and $\M^\infty(\PS)$. For each action $a \in \A$, the transition dynamics $\tilde{P}_a$ is constrained to have a rank of at most $k$. The red dashed line represents the performance of the optimal environment policy. \textbf{(b)} Trajectories starting from the bottom-right state (red dot) sampled from the optimal environment policy in both the environment and a PVE model. Note the numerous diagonal transitions in the PVE model which are not permitted in the environment.}
\label{fig:env_pve_comparison}
\end{figure}
Model classes containing many models with optimal planning performance are particularly advantageous when the set of models that an agent can represent is restricted, since the larger the set of suitable models the greater the chance of an overlap between this set and the set of models representable by the agent.
Proposition~\ref{prop:fpve_limiting_model_class_optimal2} and Corollary~\ref{prop:fpve_limiting_model_class_optimal} compared $\M^\infty(\PS)$ and $\M^\infty(\PSdet)$, showing that, although any model in either class is sufficient for planning, $\M^\infty(\PS) \subseteq \M^\infty(\PSdet)$. This suggests that it might be better to learn a model in $\M^\infty(\PSdet)$ when the agent has limited capacity. We illustrate that this is indeed the case in Figure~\ref{fig:env_pve_comparison}b. We  progressively restrict the space of models that the agent can represent and attempt to learn models in either $\M^\infty(\PS)$ or $\M^\infty(\PSdet)$. Indeed, we find that the larger class, $\M^\infty(\PSdet)$, yields superior planning performance as agent capacity decreases. 

Given their importance, we provide intuition on the ways that \textit{individual} PVE models differ from the environment. In Figure~\ref{fig:env_pve_comparison}a we compare trajectories starting at the same initial state (denoted by a red-circle) from the optimal environment policy in both the environment and in a randomly sampled model from $\M^\infty(\PS)$. In the PVE model there are numerous diagonal transitions not permitted by the environment. Note that while the PVE model has very different dynamics than the environment, these differences must ``balance out'', as it still has the same values under any policy as the environment.

\begin{wrapfigure}{R}{0.5\textwidth}
    \centering
    \includegraphics[width=0.5\textwidth]{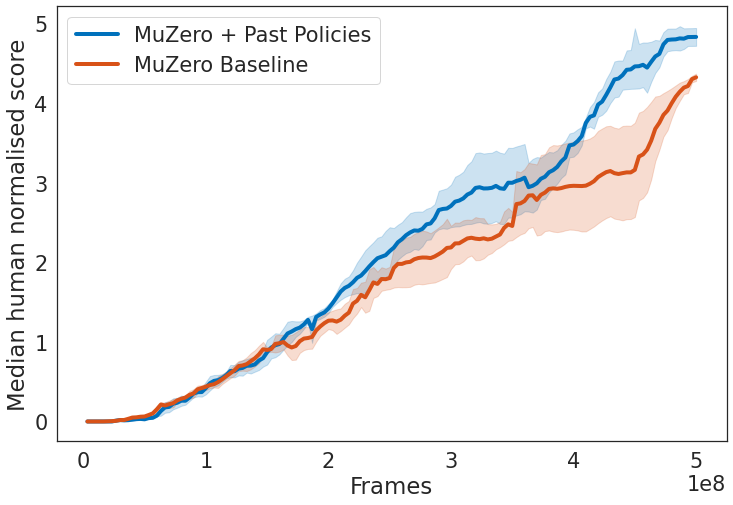}
    \caption{Comparison of our proposed modification to MuZero with an unmodified baseline.}
    \label{fig:muzero_plot}
\end{wrapfigure}

In Section~\ref{sec:muzero_connection} we showed that minimizing Muzero's loss function is analogous to minimizing an upper-bound on a PVE loss~\eqref{eq:pve_loss} with respect to the agent's current policy $\pi$---which corresponds to finding a model in $\M^\infty(\Pi)$ where $\Pi = \{ \pi \}$. Note that our guarantee on the performance of PVE models (Corollary~\ref{prop:fpve_limiting_model_class_optimal}) holds when $\Pi$ contains all deterministic policies. While it is not feasible to enforce $\Pi = \PSdet$, we can use previously seen policies by augmenting the MuZero algorithm with a buffer of past policies and their approximate value functions (we do so by periodically storing the corresponding parameters).
We can then add an additional loss to MuZero with the form of the original value loss but using the past value functions.
We still use sampled rewards to construct value targets, but use the stored policies to compute off-policy corrections using V-trace \cite{espeholt2018impala}.

To test this proposal we use an on policy (i.e., without a replay buffer) implementation of MuZero run for 500M frames (as opposed to 20B frames in the online result of~\citep{schrittwieser2020mastering}) on the Atari 57 benchmark and find that using our additional loss yields an advantage in the human normalised median performance shown in Figure \ref{fig:muzero_plot}.
MuZero's data efficiency can also be improved with the aid of a replay buffer of trajectories from past policies \citep{schrittwieser2021online}, which may also capture some of the advantages of expanding the set of policies used for PVE.

\section{Related work}
Our work is closely related to the value-aware model learning (VAML, IterVAML, \citep{farahmand2017value, farahmand2018iterative}) which learns models to minimize the discrepancy between their own Bellman optimality operators and the environment's on the optimal value function---an a priori unknown quantity. To handle this VAML specifies a family of potential value functions and minimizes the worst-case discrepancy across them, whereas IterVAML minimizes the discrepancy with respect to model's current estimate of the value function in a value iteration inspired scheme. PVE and the VAML family are complementary works, with VAML addressing its induced optimization problems and PVE addressing its induced model classes. 
Both, however, advocate for learning models with their eventual use in mind---a view that is aligned with many criticisms of the maximum likelihood objective for model learning \citep{farahmand2017value,joseph2013reinforcement, ayoub2020model, modi2020sample}).

It is worth mentioning the relationship between PVE and TD models \citep{sutton1995tdmodels} which, for a given policy, defines any $R \in \mathbb{R}^{|\S|}$ and $P \in \mathbb{R}^{|\S| \times |\S|}$ with $\lim_{k\to\infty} P^k = 0$ as a \textit{valid model} if $V = R + P^\top V$ where $V \in \mathbb{R}^{|\S|}$ represents $v_\pi$. Clearly all models in $\M^\infty(\{ \pi \})$ are valid models, however, since $P$ is not necessarily a transition matrix, the converse does not hold. While TD models are restricted to prediction rather than control, their generality warrants further inquiry.


Order-$k$ and PVE model classes form equivalences between MDPs and thus can be situated among other equivalence notions which can be formulated as state-aggregations \citep{dean1997model,poupart2003value,poupart2013value,givan2003equivalence,ravindran2004approximate,ferns2004metrics,spencer2009toward, taylor2008bounding, castro2020scalable, van2020plannable}. As pointed out by \citet{grimm2020value}, the interaction between arbitrary state-aggregation and models can be captured with special cases of order-one VE. Our extension of higher-order VEs potentially offers the possibility of ``blending'' existing notions of state aggregation with PVE.

A notable instance of state-aggregation is bisimulation~\citep{milner1989communication}, which uses a relation to aggregate states that have the same immediate rewards and transition dynamics into other aggregated states. Bisimulation metrics~\citep{ferns2004metrics} provide smooth measures of how closely pairs of states satisfy bisimulation relations. These concepts have become increasingly popular in deep reinforcement learning where they are used to guide the learning of effective representations~\citep{zhang2020learning, zhang2019learning, gelada2019deepmdp, agarwal2021contrastive}.
Although both bisimulation and PVE provide direction for learning internal aspects of an agent, they are fundamentally different in their purview---bisimulation concerns representation learning, while PVE concerns the learning of models given a representation of state.

Beyond bisimulation, representation learning has a wide literature  \citep{watter2015embed,igl2018deep,corneil2018efficient,franccois2019combined,biza2020learning} including several modern works which explicitly study the conjunction of model learning with state representation \citep{zhang2019solar,zhang2019learning,gelada2019deepmdp}. These are further complemented by efforts to learn state representations and models jointly in the service of value-based planning~\citep{farquhar2018treeqn,silver2017predictron,oh2017value,hessel2021muesli,schrittwieser2020mastering,tamar2016value}. 

\section{Conclusion and future work}
We extended the value equivalence principle by defining a spectrum of order-$k$ VE sets in which models induce the same $k$-step Bellman operators as the environment. We then explored the topology of the resulting equivalence classes and defined the limiting class when $k \to \infty$ as PVE. If a model is PVE to the environment with respect to a set of policies $\Pi$, then the value functions of all policies in $\Pi$ are the same in the environment and the model. The fact that PVE classes can be defined using only a set of policies eliminates the need for specifying a set of functions to induce VE---resolving a fundamental issue left open by \citet{grimm2020value}. Importantly, we showed that being PVE with respect to all deterministic policies is sufficient for a model to plan optimally in the environment. In the absence of additional information, this is the largest possible VE class that yields optimal planning. On the practical side, we showed how the MuZero algorithm can be understood as minimizing an upper bound on a PVE loss, and leveraged this insight to improve the algorithm's performance.

Though our efforts have advanced the understanding of value equivalence and proven useful algorithmically, there is still work to be done in developing a VE theory whose assumptions hold in practice.
This remaining work can be broadly grouped into two areas (1) understanding the role of approximation in VE and (2) establishing performance guarantees for VE models with arbitrary sets of policies and functions. We leave these as future work.

\newpage 

\noindent\textbf{Acknowledgements}

We thank Angelos Filos and Sonya Kotov for many thought-provoking discussions.
Christopher Grimm's work was made possible by the support of the Lifelong Learning Machines (L2M) grant from the Defense Advanced Research Projects Agency. Any opinions, findings, conclusions, or recommendations expressed here are those of the authors and do not necessarily reflect the views of the sponsors.

\bibliographystyle{plainnat}
\bibliography{references.bib}

\begin{thebibliography}{44}
\providecommand{\natexlab}[1]{#1}
\providecommand{\url}[1]{\texttt{#1}}
\expandafter\ifx\csname urlstyle\endcsname\relax
  \providecommand{\doi}[1]{doi: #1}\else
  \providecommand{\doi}{doi: \begingroup \urlstyle{rm}\Url}\fi

\bibitem[Agarwal et~al.(2021)Agarwal, Machado, Castro, and
  Bellemare]{agarwal2021contrastive}
Rishabh Agarwal, Marlos~C Machado, Pablo~Samuel Castro, and Marc~G Bellemare.
\newblock Contrastive behavioral similarity embeddings for generalization in
  reinforcement learning.
\newblock \emph{arXiv preprint arXiv:2101.05265}, 2021.

\bibitem[Ayoub et~al.(2020)Ayoub, Jia, Szepesvari, Wang, and
  Yang]{ayoub2020model}
Alex Ayoub, Zeyu Jia, Csaba Szepesvari, Mengdi Wang, and Lin Yang.
\newblock Model-based reinforcement learning with value-targeted regression.
\newblock In \emph{International Conference on Machine Learning}, pages
  463--474. PMLR, 2020.

\bibitem[Bellemare et~al.(2013)Bellemare, Naddaf, Veness, and
  Bowling]{bellemare2013arcade}
Marc~G Bellemare, Yavar Naddaf, Joel Veness, and Michael Bowling.
\newblock The arcade learning environment: An evaluation platform for general
  agents.
\newblock \emph{Journal of Artificial Intelligence Research}, 47:\penalty0
  253--279, 2013.

\bibitem[Biza et~al.(2020)Biza, Platt, van~de Meent, and
  Wong]{biza2020learning}
Ondrej Biza, Robert Platt, Jan-Willem van~de Meent, and Lawson~LS Wong.
\newblock Learning discrete state abstractions with deep variational inference.
\newblock \emph{arXiv preprint arXiv:2003.04300}, 2020.

\bibitem[Castro(2020)]{castro2020scalable}
Pablo~Samuel Castro.
\newblock Scalable methods for computing state similarity in deterministic
  markov decision processes.
\newblock In \emph{Proceedings of the AAAI Conference on Artificial
  Intelligence}, volume~34, pages 10069--10076, 2020.

\bibitem[Corneil et~al.(2018)Corneil, Gerstner, and Brea]{corneil2018efficient}
Dane Corneil, Wulfram Gerstner, and Johanni Brea.
\newblock Efficient model-based deep reinforcement learning with variational
  state tabulation.
\newblock In \emph{International Conference on Machine Learning}, pages
  1049--1058. PMLR, 2018.

\bibitem[Coulom(2006)]{coulom2006efficient}
R{\'e}mi Coulom.
\newblock Efficient selectivity and backup operators in monte-carlo tree
  search.
\newblock In \emph{International conference on computers and games}, pages
  72--83. Springer, 2006.

\bibitem[Dean and Givan(1997)]{dean1997model}
Thomas Dean and Robert Givan.
\newblock Model minimization in markov decision processes.
\newblock In \emph{AAAI/IAAI}, pages 106--111, 1997.

\bibitem[Espeholt et~al.(2018)Espeholt, Soyer, Munos, Simonyan, Mnih, Ward,
  Doron, Firoiu, Harley, Dunning, et~al.]{espeholt2018impala}
Lasse Espeholt, Hubert Soyer, Remi Munos, Karen Simonyan, Vlad Mnih, Tom Ward,
  Yotam Doron, Vlad Firoiu, Tim Harley, Iain Dunning, et~al.
\newblock Impala: Scalable distributed deep-rl with importance weighted
  actor-learner architectures.
\newblock In \emph{International Conference on Machine Learning}, pages
  1407--1416. PMLR, 2018.

\bibitem[Farahmand(2018)]{farahmand2018iterative}
Amir{-}massoud Farahmand.
\newblock Iterative value-aware model learning.
\newblock In \emph{Advances in Neural Information Processing Systems
  ({NeurIPS})}, pages 9090--9101, 2018.

\bibitem[Farahmand et~al.(2017)Farahmand, Barreto, and
  Nikovski]{farahmand2017value}
Amir-Massoud Farahmand, Andr\'{e} Barreto, and Daniel Nikovski.
\newblock {Value-Aware Loss Function for Model-based Reinforcement Learning}.
\newblock In \emph{Proceedings of the International Conference on Artificial
  Intelligence and Statistics ({AISTATS})}, volume~54, pages 1486--1494, 2017.

\bibitem[Farquhar et~al.(2018)Farquhar, Rockt{\"a}schel, Igl, and
  Whiteson]{farquhar2018treeqn}
G~Farquhar, T~Rockt{\"a}schel, M~Igl, and S~Whiteson.
\newblock Treeqn and atreec: Differentiable tree-structured models for deep
  reinforcement learning.
\newblock In \emph{6th International Conference on Learning Representations,
  ICLR 2018-Conference Track Proceedings}, volume~6. ICLR, 2018.

\bibitem[Ferns et~al.(2004)Ferns, Panangaden, and Precup]{ferns2004metrics}
Norm Ferns, Prakash Panangaden, and Doina Precup.
\newblock Metrics for finite markov decision processes.
\newblock In \emph{UAI}, volume~4, pages 162--169, 2004.

\bibitem[Fran{\c{c}}ois-Lavet et~al.(2019)Fran{\c{c}}ois-Lavet, Bengio, Precup,
  and Pineau]{franccois2019combined}
Vincent Fran{\c{c}}ois-Lavet, Yoshua Bengio, Doina Precup, and Joelle Pineau.
\newblock Combined reinforcement learning via abstract representations.
\newblock In \emph{Proceedings of the AAAI Conference on Artificial
  Intelligence}, volume~33, pages 3582--3589, 2019.

\bibitem[Gelada et~al.(2019)Gelada, Kumar, Buckman, Nachum, and
  Bellemare]{gelada2019deepmdp}
Carles Gelada, Saurabh Kumar, Jacob Buckman, Ofir Nachum, and Marc~G Bellemare.
\newblock Deepmdp: Learning continuous latent space models for representation
  learning.
\newblock In \emph{International Conference on Machine Learning}, pages
  2170--2179. PMLR, 2019.

\bibitem[Givan et~al.(2003)Givan, Dean, and Greig]{givan2003equivalence}
Robert Givan, Thomas Dean, and Matthew Greig.
\newblock Equivalence notions and model minimization in markov decision
  processes.
\newblock \emph{Artificial Intelligence}, 147\penalty0 (1-2):\penalty0
  163--223, 2003.

\bibitem[Grimm et~al.(2020)Grimm, Barreto, Singh, and Silver]{grimm2020value}
Christopher Grimm, Andre Barreto, Satinder Singh, and David Silver.
\newblock The value equivalence principle for model-based reinforcement
  learning.
\newblock \emph{Advances in Neural Information Processing Systems}, 33, 2020.

\bibitem[Hessel et~al.(2021{\natexlab{a}})Hessel, Danihelka, Viola, Guez,
  Schmitt, Sifre, Weber, Silver, and van Hasselt]{hessel2021muesli}
Matteo Hessel, Ivo Danihelka, Fabio Viola, Arthur Guez, Simon Schmitt, Laurent
  Sifre, Theophane Weber, David Silver, and Hado van Hasselt.
\newblock Muesli: Combining improvements in policy optimization.
\newblock \emph{arXiv preprint arXiv:2104.06159}, 2021{\natexlab{a}}.

\bibitem[Hessel et~al.(2021{\natexlab{b}})Hessel, Kroiss, Clark, Kemaev, Quan,
  Keck, Viola, and van Hasselt]{hessel2021b}
Matteo Hessel, Manuel Kroiss, Aidan Clark, Iurii Kemaev, John Quan, Thomas
  Keck, Fabio Viola, and Hado van Hasselt.
\newblock \href{https://arxiv.org/abs/2104.06272}{Podracer architectures for
  scalable Reinforcement Learning}.
\newblock \emph{CoRR}, abs/2104.06272, 2021{\natexlab{b}}.

\bibitem[Igl et~al.(2018)Igl, Zintgraf, Le, Wood, and Whiteson]{igl2018deep}
Maximilian Igl, Luisa Zintgraf, Tuan~Anh Le, Frank Wood, and Shimon Whiteson.
\newblock Deep variational reinforcement learning for pomdps.
\newblock In \emph{International Conference on Machine Learning}, pages
  2117--2126. PMLR, 2018.

\bibitem[Joseph et~al.(2013)Joseph, Geramifard, Roberts, How, and
  Roy]{joseph2013reinforcement}
Joshua Joseph, Alborz Geramifard, John~W Roberts, Jonathan~P How, and Nicholas
  Roy.
\newblock Reinforcement learning with misspecified model classes.
\newblock In \emph{2013 IEEE International Conference on Robotics and
  Automation}, pages 939--946. IEEE, 2013.

\bibitem[Milner()]{milner1989communication}
Robin Milner.
\newblock \emph{Communication and concurrency}, volume~84.

\bibitem[Modi et~al.(2020)Modi, Jiang, Tewari, and Singh]{modi2020sample}
Aditya Modi, Nan Jiang, Ambuj Tewari, and Satinder Singh.
\newblock Sample complexity of reinforcement learning using linearly combined
  model ensembles.
\newblock In \emph{International Conference on Artificial Intelligence and
  Statistics}, pages 2010--2020. PMLR, 2020.

\bibitem[Oh et~al.(2017)Oh, Singh, and Lee]{oh2017value}
Junhyuk Oh, Satinder Singh, and Honglak Lee.
\newblock Value prediction network.
\newblock In \emph{Advances in Neural Information Processing Systems}, pages
  6118--6128, 2017.

\bibitem[Poupart and Boutilier(2013)]{poupart2013value}
Pascal Poupart and Craig Boutilier.
\newblock Value-directed belief state approximation for {POMDPs}.
\newblock \emph{CoRR}, abs/1301.3887, 2013.
\newblock URL \url{http://arxiv.org/abs/1301.3887}.

\bibitem[Poupart et~al.(2003)Poupart, Boutilier, et~al.]{poupart2003value}
Pascal Poupart, Craig Boutilier, et~al.
\newblock Value-directed compression of pomdps.
\newblock \emph{Advances in neural information processing systems}, pages
  1579--1586, 2003.

\bibitem[Puterman(1994)]{puterman94markov}
Martin~L. Puterman.
\newblock \emph{Markov Decision Processes---Discrete Stochastic Dynamic
  Programming}.
\newblock John Wiley \& Sons, Inc., 1994.

\bibitem[Ravindran and Barto(2004)]{ravindran2004approximate}
Balaraman Ravindran and Andrew~G Barto.
\newblock Approximate homomorphisms: A framework for non-exact minimization in
  markov decision processes.
\newblock 2004.

\bibitem[Russell and Norvig(2003)]{russel2003artificial}
Stuart~J. Russell and Peter Norvig.
\newblock \emph{Artificial Intelligence: A Modern Approach}.
\newblock Pearson Education, 3 edition, 2003.

\bibitem[Schrittwieser et~al.(2019)Schrittwieser, Antonoglou, Hubert, Simonyan,
  Sifre, Schmitt, Guez, Lockhart, Hassabis, Graepel,
  et~al.]{schrittwieser2019mastering}
Julian Schrittwieser, Ioannis Antonoglou, Thomas Hubert, Karen Simonyan,
  Laurent Sifre, Simon Schmitt, Arthur Guez, Edward Lockhart, Demis Hassabis,
  Thore Graepel, et~al.
\newblock Mastering atari, go, chess and shogi by planning with a learned
  model.
\newblock \emph{arXiv preprint arXiv:1911.08265}, 2019.

\bibitem[Schrittwieser et~al.(2020)Schrittwieser, Antonoglou, Hubert, Simonyan,
  Sifre, Schmitt, Guez, Lockhart, Hassabis, Graepel,
  et~al.]{schrittwieser2020mastering}
Julian Schrittwieser, Ioannis Antonoglou, Thomas Hubert, Karen Simonyan,
  Laurent Sifre, Simon Schmitt, Arthur Guez, Edward Lockhart, Demis Hassabis,
  Thore Graepel, et~al.
\newblock Mastering atari, go, chess and shogi by planning with a learned
  model.
\newblock \emph{Nature}, 588\penalty0 (7839):\penalty0 604--609, 2020.

\bibitem[Schrittwieser et~al.(2021)Schrittwieser, Hubert, Mandhane, Barekatain,
  Antonoglou, and Silver]{schrittwieser2021online}
Julian Schrittwieser, Thomas Hubert, Amol Mandhane, Mohammadamin Barekatain,
  Ioannis Antonoglou, and David Silver.
\newblock Online and offline reinforcement learning by planning with a learned
  model.
\newblock \emph{arXiv preprint arXiv:2104.06294}, 2021.

\bibitem[Silver et~al.(2017)Silver, van Hasselt, Hessel, Schaul, Guez, Harley,
  Dulac-Arnold, Reichert, Rabinowitz, Barreto, et~al.]{silver2017predictron}
David Silver, Hado van Hasselt, Matteo Hessel, Tom Schaul, Arthur Guez, Tim
  Harley, Gabriel Dulac-Arnold, David Reichert, Neil Rabinowitz, Andre Barreto,
  et~al.
\newblock The predictron: End-to-end learning and planning.
\newblock In \emph{Proceedings of the 34th International Conference on Machine
  Learning-Volume 70}, pages 3191--3199. JMLR. org, 2017.

\bibitem[Spencer et~al.(2009)Spencer, Thomas, and
  McClelland]{spencer2009toward}
John~P Spencer, Michael~SC Thomas, and JL~McClelland.
\newblock Toward a unified theory of development.
\newblock \emph{JP Spencer, MSC Thomas, \& JL McClelland (Eds.)}, pages
  86--118, 2009.

\bibitem[Sutton(1995)]{sutton1995tdmodels}
Richard~S. Sutton.
\newblock {TD} models: Modeling the world at a mixture of time scales.
\newblock In \emph{Proceedings of the Twelfth International Conference on
  Machine Learning}, pages 531--539, 1995.

\bibitem[Sutton and Barto(2018)]{sutton2018reinforcement}
Richard~S. Sutton and Andrew~G. Barto.
\newblock \emph{Reinforcement Learning: An Introduction}.
\newblock MIT Press, 2018.
\newblock URL
  \url{https://mitpress.mit.edu/books/reinforcement-learning-second-edition}.
\newblock 2nd edition.

\bibitem[Szepesv{\'a}ri(2010)]{szepesvari2010algorithms}
Csaba Szepesv{\'a}ri.
\newblock \emph{Algorithms for Reinforcement Learning}.
\newblock Synthesis Lectures on Artificial Intelligence and Machine Learning.
  Morgan {\&} Claypool Publishers, 2010.

\bibitem[Tamar et~al.(2016)Tamar, Wu, Thomas, Levine, and
  Abbeel]{tamar2016value}
Aviv Tamar, Yi~Wu, Garrett Thomas, Sergey Levine, and Pieter Abbeel.
\newblock Value iteration networks.
\newblock In \emph{Advances in Neural Information Processing Systems}, pages
  2154--2162, 2016.

\bibitem[Taylor et~al.(2008)Taylor, Precup, and Panagaden]{taylor2008bounding}
Jonathan Taylor, Doina Precup, and Prakash Panagaden.
\newblock Bounding performance loss in approximate mdp homomorphisms.
\newblock \emph{Advances in Neural Information Processing Systems},
  21:\penalty0 1649--1656, 2008.

\bibitem[van~der Pol et~al.(2020)van~der Pol, Kipf, Oliehoek, and
  Welling]{van2020plannable}
Elise van~der Pol, Thomas Kipf, Frans~A Oliehoek, and Max Welling.
\newblock Plannable approximations to mdp homomorphisms: Equivariance under
  actions.
\newblock In \emph{Proceedings of the 19th International Conference on
  Autonomous Agents and MultiAgent Systems}, pages 1431--1439, 2020.

\bibitem[Watter et~al.(2015)Watter, Springenberg, Boedecker, and
  Riedmiller]{watter2015embed}
Manuel Watter, Jost~Tobias Springenberg, Joschka Boedecker, and Martin
  Riedmiller.
\newblock Embed to control: a locally linear latent dynamics model for control
  from raw images.
\newblock In \emph{Proceedings of the 28th International Conference on Neural
  Information Processing Systems-Volume 2}, pages 2746--2754, 2015.

\bibitem[Zhang et~al.(2019{\natexlab{a}})Zhang, Lipton, Pineda,
  Azizzadenesheli, Anandkumar, Itti, Pineau, and Furlanello]{zhang2019learning}
Amy Zhang, Zachary~C Lipton, Luis Pineda, Kamyar Azizzadenesheli, Anima
  Anandkumar, Laurent Itti, Joelle Pineau, and Tommaso Furlanello.
\newblock Learning causal state representations of partially observable
  environments.
\newblock \emph{arXiv preprint arXiv:1906.10437}, 2019{\natexlab{a}}.

\bibitem[Zhang et~al.(2020)Zhang, McAllister, Calandra, Gal, and
  Levine]{zhang2020learning}
Amy Zhang, Rowan McAllister, Roberto Calandra, Yarin Gal, and Sergey Levine.
\newblock Learning invariant representations for reinforcement learning without
  reconstruction.
\newblock \emph{arXiv preprint arXiv:2006.10742}, 2020.

\bibitem[Zhang et~al.(2019{\natexlab{b}})Zhang, Vikram, Smith, Abbeel, Johnson,
  and Levine]{zhang2019solar}
Marvin Zhang, Sharad Vikram, Laura Smith, Pieter Abbeel, Matthew Johnson, and
  Sergey Levine.
\newblock Solar: Deep structured representations for model-based reinforcement
  learning.
\newblock In \emph{International Conference on Machine Learning}, pages
  7444--7453. PMLR, 2019{\natexlab{b}}.

\end{thebibliography}

\newpage 
\section*{Checklist}

\begin{enumerate}

\item For all authors...
\begin{enumerate}
  \item Do the main claims made in the abstract and introduction accurately reflect the paper's contributions and scope?
    \answerYes{}
  \item Did you describe the limitations of your work?
    \answerYes{}
  \item Did you discuss any potential negative societal impacts of your work?
    \answerNA{}
  \item Have you read the ethics review guidelines and ensured that your paper conforms to them?
    \answerYes{}
\end{enumerate}

\item If you are including theoretical results...
\begin{enumerate}
  \item Did you state the full set of assumptions of all theoretical results?
    \answerYes{}
	\item Did you include complete proofs of all theoretical results?
    \answerYes{See appendix.}
\end{enumerate}

\item If you ran experiments...
\begin{enumerate}
  \item Did you include the code, data, and instructions needed to reproduce the main experimental results (either in the supplemental material or as a URL)?
    \answerYes{Code for the illustrative experiments is available at a URL provided in Appendix~\ref{sec:experimental_details}. Sufficient detail for reproducability is provided in the same section.}
  \item Did you specify all the training details (e.g., data splits, hyperparameters, how they were chosen)?
    \answerYes{} 
	\item Did you report error bars (e.g., with respect to the random seed after running experiments multiple times)?
    \answerYes{}
	\item Did you include the total amount of compute and the type of resources used (e.g., type of GPUs, internal cluster, or cloud provider)?
    \answerYes{}
\end{enumerate}

\item If you are using existing assets (e.g., code, data, models) or curating/releasing new assets...
\begin{enumerate}
  \item If your work uses existing assets, did you cite the creators?
    \answerYes{}
  \item Did you mention the license of the assets?
    \answerNA{}
  \item Did you include any new assets either in the supplemental material or as a URL?
    \answerYes{}
  \item Did you discuss whether and how consent was obtained from people whose data you're using/curating?
    \answerNA{}
  \item Did you discuss whether the data you are using/curating contains personally identifiable information or offensive content?
    \answerNA{}
\end{enumerate}

\item If you used crowdsourcing or conducted research with human subjects...
\begin{enumerate}
  \item Did you include the full text of instructions given to participants and screenshots, if applicable?
    \answerNA{}
  \item Did you describe any potential participant risks, with links to Institutional Review Board (IRB) approvals, if applicable?
    \answerNA{}
  \item Did you include the estimated hourly wage paid to participants and the total amount spent on participant compensation?
    \answerNA{}
\end{enumerate}

\end{enumerate}

\newpage 
\appendix

\section{Appendix}
\label{appendix}

\subsection{Illustrative MDPs}
\label{sec:mdps}
Several proofs in \ref{sec:proofs} rely on constructing special MDPs to serve as examples or counterexamples. 
We reserve this section to describe these MDPs for later reference.

\subsubsection{Ring and false-ring MDPs}
\label{sec:rings}
We consider a simple $n$-state, $1$ action ``ring'' MDP (Figure~\ref{fig:rings}) denoted $m^n_\circ = (r, p)$ where:
\begin{equation}
\begin{aligned}
r(s_i) = g(i)\ \forall i \in [n], \hspace{20pt} p(s_{i+1} | s_{i}) = 1 \hspace{4pt} \forall i \in [n-1] \hspace{4pt} \text{ and } \hspace{4pt} p(s_1 | s_n) = 1
\end{aligned}
\end{equation}
where $g : i \mapsto \mathbb{R}$ is some function that defines the reward from transitioning away from state $i$. Since $|\A| = 1$ we omit actions from the reward and transition dynamics. 

For each ring MDP and function $g$ we additionally construct a corresponding ``false-ring'' MDP (Figure~\ref{fig:rings}) with the same state and actions spaces as $m^n_\circ$ but with states that only self-transition and with rewards designed to mimic the discounted $n$-step returns on Ring MDPs. 
We represent these as $\tilde{m}^n_\circ = (\tilde{r}, \tilde{p})$ where
\begin{equation}
\begin{aligned}
\tilde{r}(s_i) = \frac{r^n(s_i)}{\sum_{t=0}^{n-1} \gamma^t}, \hspace{20pt} \tilde{p}(s_i | s_i) = 1
\end{aligned}
\end{equation}
and $r^n(s_i)$ denotes the discounted $n$-step return starting from $s_i$ in $m^n_\circ$. Note that the discounted $n$-step return of an $n$-state false-ring MDP is the same as that of an $n$-state ring MDP. 

We now provide some basic results about pairs of ring and false-ring MDPs that we will use periodically in our proofs.

\begin{lemma}
\label{lemma:rings}
For any $n \in \mathbb{Z}^+ \cup \{ \infty \}$ if we treat the ring MDP $m^n_\circ$ as the environment and assume $\tilde{m}^n_\circ \in \mathcal{M}$ it follows that
\begin{equation}
\tilde{m}^n_\circ \in \M^n(\PS, \FS).
\end{equation}
when $n < \infty$ and
\begin{equation}
\tilde{m}^n_\circ \in \M^\infty(\PS)
\end{equation}
when $n = \infty$.
\end{lemma}

\begin{proof}
First we note that, since ring and false-ring MDPs only have one action, we can write $\PS = \{ \pi \}$ where $\pi$ takes this action at all states.
We first consider the case when $n < \infty$, noting that that both MDPs are deterministic and that for any state $s$, performing $n$ transitions will always return to $s$.
We now consider an application of $n$-step Bellman operator of the false-ring model to an arbitrary function $v \in \FS$:
\begin{equation}
\begin{aligned}
\tilde{\mathcal{T}}_\pi^n v(s) &= \tilde{r}(s)(\sum_{t=0}^{n-1} \gamma^t) + \gamma^n v(s) \\ 
&= r^n(s) (\sum_{t=0}^{n-1} \gamma^t)^{-1} (\sum_{t=0}^{n-1} \gamma^t) + \gamma^n v(s) \\ 
&= r^n(s) + \gamma^n v(s) \\
&= \mathcal{T}_\pi^n v(s)
\end{aligned}
\end{equation}
implying that $\tilde{m}^n_\circ \in \M^n(\PS, \FS)$ as needed. We now consider the case when $n = \infty$:
here, we note that for any state $s \in \S$:
\begin{equation}
\tilde{r}(s) = \frac{r^\infty(s)}{ \sum_{t=0}^\infty \gamma^t} = (1 - \gamma) v_\pi(s).
\end{equation}
we can then write:
\begin{equation}
\tilde{v}_\pi(s) = \sum_{t=0}^\infty \gamma^t \tilde{r}(s) = (1 - \gamma)^{-1} (1 - \gamma) v_\pi(s) = v_\pi(s)
\end{equation}
since $\tilde{m}^\infty_\circ$ only self-transitions at each state. This shows that $\tilde{m}^\infty_\circ \in \M^\infty(\Pi)$ as needed.

\end{proof}

\begin{lemma}
\label{lemma:rings_k_step_update}
Fix any $k, K \in \mathbb{Z}^+ \cup \{ \infty \}$ with $k < K$ and let $f : \S \mapsto \mathbb{R}$ be any constant function. Let $m = m^K_\circ$ and $\tilde{m} = \tilde{m}^K_\circ$. For any $\gamma \in (0, 1)$ it follows that
\begin{equation}
\mathcal{T}_\pi^k f(s_1) \neq \tilde{\mathcal{T}}_\pi^k f(s_1)
\end{equation}
where $m^K_\circ$ and $\tilde{m}^K_\circ$ are $K$-state ring and false-ring MDPs with $g(i) = \boldsymbol{1}\{ i \in [1, k] \}$. 
\end{lemma}

\begin{proof}
We begin by examining the $k$-step Bellman operator and Bellman fixed-point under the ring $m^K_\circ$:
\begin{equation}
\label{eq:ring_update}
\mathcal{T}_\pi^k f(s_1) = r^k(s_1) + \gamma^k f(s_1) = r^K(s_1) + \gamma^k f(s_1)
\end{equation}
where the second equality follows from the fact that $g$ ensures that no reward is received after the first $k$ steps from $s_1$. 

Next we examine the corresponding $k$-step Bellman operator under the false-ring $\tilde{m}^K_\circ$:
\begin{equation}
\label{eq:false_ring_update}
\tilde{\mathcal{T}}_\pi^k f(s_1) = \tilde{r}^k(s_1) + \gamma^k f(s_1) = r^K(s_1)(\sum_{t=0}^{K-1} \gamma^t)^{-1} \sum_{t=0}^{k-1} \gamma^t + \gamma^k f(s_1)
\end{equation}
where the second equality follows from the construction of $K$-step false-ring MDPs to match the $K$-step returns of their corresponding ring MDP.

Taken together Eqs.~(\ref{eq:ring_update}-\ref{eq:false_ring_update}) imply that in order for $\mathcal{T}_\pi^k f(s_1) = \tilde{\mathcal{T}}_\pi^k f(s_1)$ it must be the case that $\sum_{t=0}^{K-1} \gamma^t = \sum_{t=0}^{k-1} \gamma^t$ which can only happen when $\gamma = 0$.
Note that these properties hold when $K = \infty$.
This completes the proof.
\end{proof}

\begin{figure}
\centering
\begin{subfigure}{.25\textwidth}
  \centering
  \includegraphics[width=\linewidth]{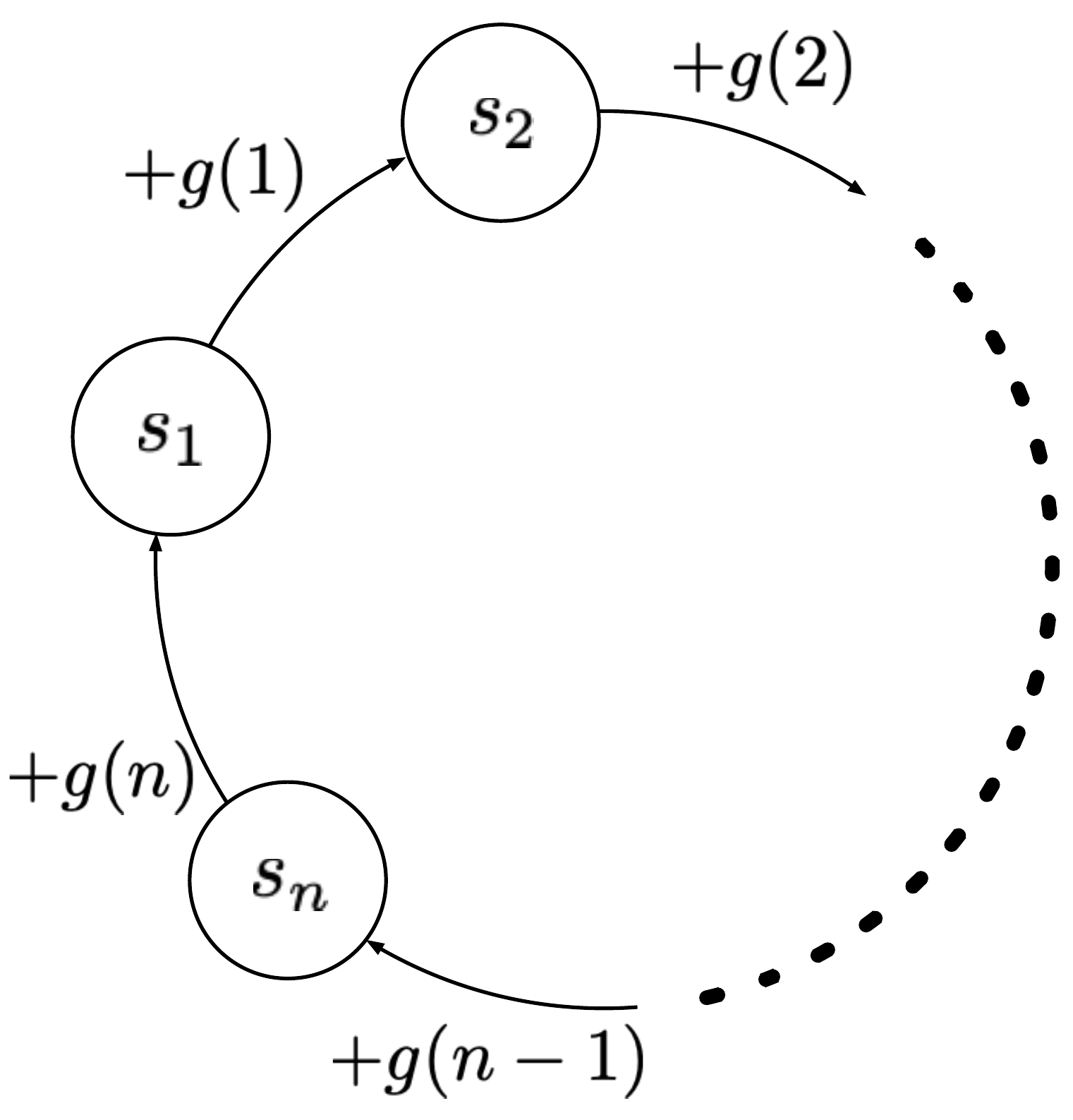}
  \caption{}
  \label{fig:sub2}
\end{subfigure}%
\begin{subfigure}{.25\textwidth}
  \centering
  \includegraphics[width=\linewidth]{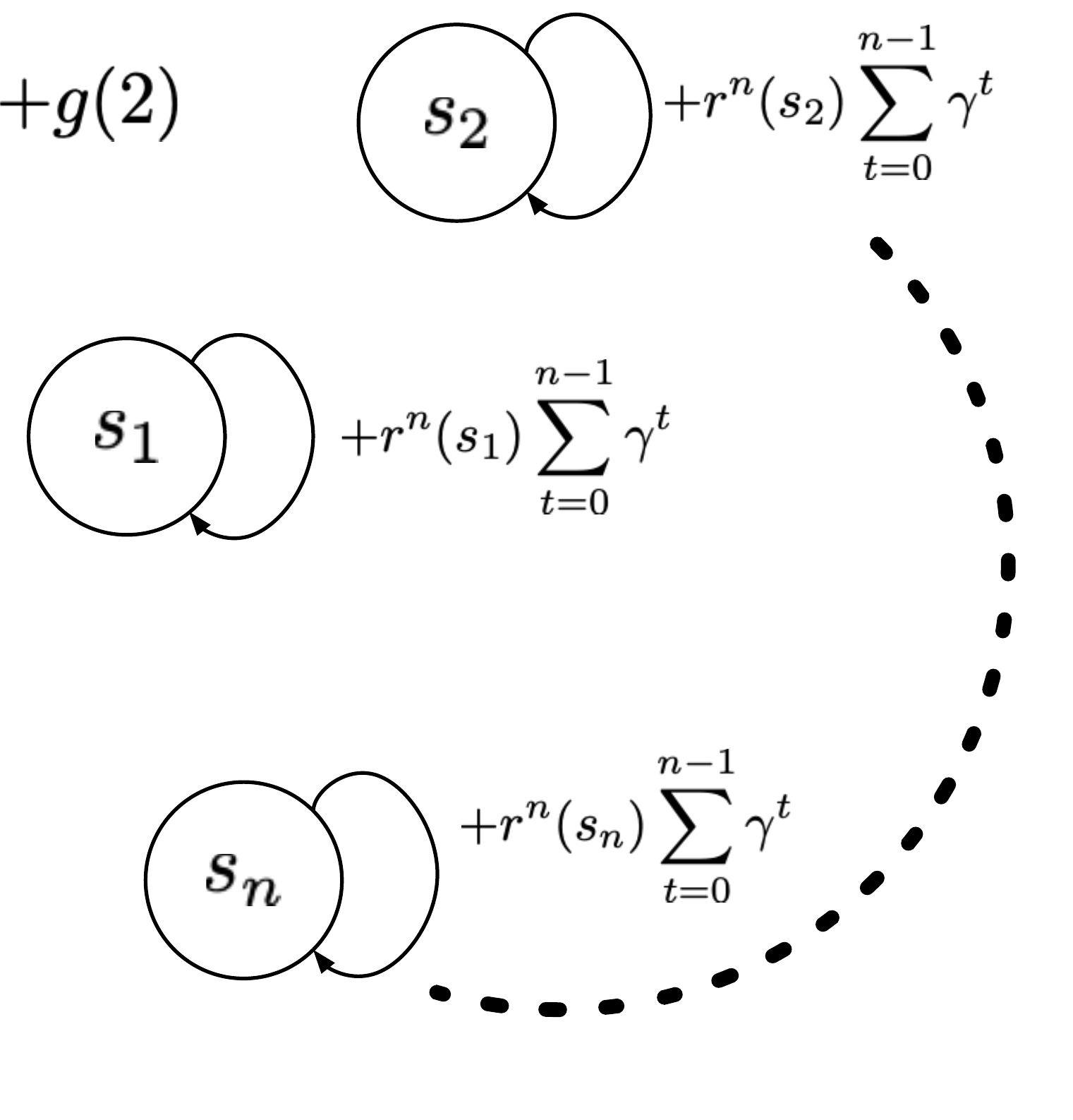}
  \caption{}
  \label{fig:sub2}
\end{subfigure}
\caption{Ring and false-ring environments with reward structure defined by $g : \mathbb{Z}^+ \mapsto \mathbb{R}$. States are numbered circles and outgoing arrows indicate possible transitions from each state. Arrows are labeled by the reward attained from performing their transition. }
\label{fig:rings}
\end{figure}

\subsection{Proofs}
\label{sec:proofs}
In this section we provide proofs of the results in the main text.



\moduloResult*
\begin{proof}
 Consider some $m \in \mathcal{M}^k(\Pi, \mathcal{V})$. 
For any $\pi \in \Pi$ and $v \in \mathcal{V}$ we know that $\tilde{\mathcal{T}}_\pi^k v = \mathcal{T}_\pi^k v$.
Since $k$ divides $K$ we know that $K = zk$ where $z \in \mathbb{Z}^+$. Hence
\begin{equation}
\label{eq:k_to_K}
\begin{aligned}
\mathcal{T}_\pi^K v = \underbrace{\mathcal{T}_\pi \cdots \mathcal{T}_\pi}_{\text{K times}} v 
= \underbrace{\mathcal{T}_\pi^k \cdots \mathcal{T}_\pi^k}_\text{z times} v 
\end{aligned}
\end{equation}

Finally since $\mathcal{V}$ is closed under Bellman updates we can write $\tilde{\mathcal{T}}_\pi^k v = \mathcal{T}_\pi^k v \in \mathcal{V}$, which allows us iteratively equate $k$-step environment and model operators on the right-hand side of Eq.~(\ref{eq:k_to_K}) to obtain:
\begin{equation}
\underbrace{\mathcal{T}_\pi^k \cdots \mathcal{T}_\pi^k}_\text{z times} v = \underbrace{\tilde{\mathcal{T}}_\pi^k \cdots \tilde{\mathcal{T}}_\pi^k}_\text{z times} v = \tilde{\mathcal{T}}_\pi^K v.
\end{equation}

This suffices to show that $m \in \mathcal{M}^K(\Pi, \mathcal{V})$ which means $\mathcal{M}^k(\Pi, \mathcal{V}) \subseteq \mathcal{M}^K(\Pi, \mathcal{V})$.

We now assume that $\V$ contains at least one constant function and $\Pi$ is non-empty and produce an instance of an environment and model class where the relation is strict. Let the environment be a $K$-state ring environment (see \ref{sec:rings}): $m^K_\circ$ with $g(i) = \boldsymbol{1}\{ i \in [1, k] \}$ and let $\M = \MS$. 
Next we introduce a model given by the corresponding false-ring MDP (see \ref{sec:rings}) $\tilde{m}^K_\circ$. 
From Lemma~\ref{lemma:rings} we have that $\tilde{m}^K_\circ \in \M^K(\Pi, \V)$.

Since there is at least one constant function $f \in \V$ we know that $\mathcal{T}_\pi^k f(s_1) \neq \tilde{\mathcal{T}}_\pi^k f(s_1)$ from Lemma~\ref{lemma:rings_k_step_update}. 
This is sufficient to show that $\tilde{m}^K_\circ \notin \M^k(\Pi, \V)$ and thus we have proven that there are instances where $\M^k(\Pi, \V) \subset \M^K(\Pi, \V)$.
\end{proof}

\FPVEDecomp*

\begin{proof}
We first note  
$\M^\infty(\Pi) = \bigcap_{\pi \in \Pi} \M^\infty(\{ \pi \})$ and consider any $m \in \M^\infty(\{ \pi \})$ for some $\pi \in \Pi$. From the definition of PVE we know $\tilde{v}_\pi = v_\pi$ and thus can say:
\begin{equation}
\begin{aligned}
&\tilde{v}_\pi = v_\pi \\
\implies &\tilde{\mathcal{T}}_\pi^k \tilde{v}_\pi = \tilde{\mathcal{T}}_\pi^k v_\pi \\
\implies &\tilde{v}_\pi = \tilde{\mathcal{T}}_\pi^k v_\pi\\
\implies &v_\pi = \tilde{\mathcal{T}}_\pi^k v_\pi \\
\implies &\mathcal{T}_\pi^k v_\pi = \tilde{\mathcal{T}}_\pi^k v_\pi 
\end{aligned}
\end{equation}
which suggests that $m \in \M^k(\{ \pi \}, \{ v^\pi \})$ and thus $\M^\infty(\{ \pi \}) \subseteq \M^k(\{ \pi \}, \{ v^\pi \})$.

We now consider any element $m \in \M^k(\{ \pi \}, \{ v^\pi \})$, and note that from the definition of order-$k$ VE we know that $\tilde{\mathcal{T}}_\pi^k v_\pi = \mathcal{T}_\pi^k v_\pi$, thus we can say:
\begin{equation}
\begin{aligned}
&\tilde{\mathcal{T}}_\pi^k v_\pi = \mathcal{T}_\pi^k v_\pi \\
\implies & \tilde{\mathcal{T}}_\pi^k v_\pi = v_\pi \\
\implies & \tilde{\mathcal{T}}_\pi^{2k} v_\pi = \tilde{\mathcal{T}}_\pi^k v_\pi \\
\implies & \tilde{\mathcal{T}}_\pi^{2k} v_\pi = v_\pi
\end{aligned}
\end{equation}
where we can repeat the process described in these implications ad-infinitum to obtain $\tilde{v}_\pi = \lim_{n \to \infty} \tilde{\mathcal{T}}_\pi^{nk} v_\pi = v_\pi$. Hence $m \in \M^\infty(\{ \pi \})$ and thus $\M^k(\{ \pi \}, \{ v^\pi \})$. 

Taken together this shows that $\M^\infty(\{ \pi \}) = \M^k(\{ \pi \}, \{ v^\pi \})$ for any $k$ and $\pi$ thus:
\begin{equation}
\M^\infty(\Pi) = \bigcap_{\pi \in \Pi} \M^\infty(\{ \pi \}) = \bigcap_{\pi \in \Pi} \M^k(\{ \pi \}, \{ v^\pi \})
\end{equation}
for any $k \in \mathbb{Z}^+$.
\end{proof}

\begin{restatable}{corollary}{limitingComparison}
\label{coro:FPVE_vs_VE_limiting}
Let $\Pi \subseteq \mathbbl{\Pi}$ and let \V\ be as in Proposition~\ref{property:limiting_model_classes} for $k \in \mathbb{Z}^+$ then we have that
$
\mathcal{M}^k(\Pi, \V) \subseteq \mathcal{M}^\infty(\Pi).
$ Moreover, if $\Pi$ is non-empty and $\V$ contains at least one constant function, then there exist environments such that $\MS^k(\Pi, \V) \subset \MS^\infty(\Pi)$
\end{restatable}
\begin{proof}
Consider some $m \in \mathcal{M}^k(\Pi, \mathcal{V})$.
From the generalization of Property~\ref{property:limiting_model_classes} we know that $m \in \mathcal{M}^{zk}(\Pi, \mathcal{V})$ for any $z \in \mathbb{Z}^+$ since $k$ divides $zk$.
Thus we know that $\tilde{\mathcal{T}}_\pi^{zk} v = \mathcal{T}_\pi^{zk} v$ for any choice of $\pi \in \Pi$, $v \in \mathcal{V}$ and $z \in \mathbb{Z}^+$.
Accordingly the expressions are equal in the limit as $z \to \infty$. 
Combining this with the fact that both $\tilde{\mathcal{T}}_\pi$ and $\mathcal{T}_\pi$ are contraction mappings, we obtain:
\begin{equation}
\tilde{v}_\pi = \lim_{z \to \infty} \tilde{\mathcal{T}}_\pi^{zk} v = \lim_{z \to \infty} \mathcal{T}_\pi^{zk} v = v_\pi 
\end{equation}
which implies $m \in \mathcal{M}^\infty(\Pi)$ and thus $\mathcal{M}^k(\Pi, \mathcal{V}) \subseteq \mathcal{M}^\infty(\Pi)$, as needed. 

Moreover, so long that $\Pi$ is nonempty and $\V$ contains some constant function $f$, we can construct a pair of $\infty$-state ring / false-ring MDPs: $m^\infty_\circ$ and $\tilde{m}^\infty_\circ$ with $g(i) = \boldsymbol{1}\{ i \in [1, k] \}$ (see \ref{sec:rings}).
By assuming that $m^\infty_\circ$ is the environment, Lemma~\ref{lemma:rings} tells us that $\tilde{m}^\infty_\circ \in \M^\infty(\Pi)$ and 
we know from Lemma~\ref{lemma:rings_k_step_update} that $\mathcal{T}_\pi^kf(s_1) \neq \tilde{\mathcal{T}}_\pi^k f(s_1)$ hence $\tilde{m}^\infty_\circ \notin \M^k(\Pi)$.
\end{proof}

\irrelevantState*

\begin{proof}
Assume that $\mathcal{M} = \MS$. Denote the environment reward and transition dynamics as $(r, p)$. 
For any value $y_0 \in \mathcal{Y}$ we consider a model $m_{y_0}$:
\begin{equation}
\begin{aligned}
&r_{m_{y_0}}((x, y), a) = r((x, y), a) \\
&p_{m_{y_0}}((x', y')| (x, y), a) = \boldsymbol{1}\{ y' = y_0 \} p(x' | (x, y), a). 
\end{aligned}
\end{equation}

We now examine the Bellman fixed-point induced by environment for any policy $\pi \in \Pi$:
\begin{equation}
\begin{aligned}
v_\pi((x, y)) &= \int_{\mathcal{A}} \pi(a | (x, y)) r((x, y), a) + \gamma \int_\mathcal{X} \int_\mathcal{Y} p((x',y') | (x, y), a) v_\pi((x', y')) dx' dy' da \\
&= \int_{\mathcal{A}} \pi(a | (x, y)) r((x, y), a) + \gamma \int_\mathcal{X} \int_\mathcal{Y} p( x' | (x, y), a) p(y' | x', (x, y), a) v_\pi(x') dx' dy' da \\
&= \int_{\mathcal{A}} \pi(a | (x, y)) r((x, y), a) + \gamma \int_\mathcal{X} p(x' | (x, y), a) v_\pi(x') dx' da.
\end{aligned}
\end{equation}

We can compare this to the Bellman operator induced by our model for the same policy:
\begin{equation}
\begin{aligned}
\tilde{\mathcal{T}}_\pi v((x, y)) &= \int_{\mathcal{A}} \pi(a | (x, y)) r((x, y), a) + \gamma \int_{\mathcal{X}} \int_{\mathcal{Y}} \boldsymbol{1}\{ y' = y_0 \} p(x' | (x, y), a) v((x, y)) dx' dy' da\\
&= \int_{\mathcal{A}} \pi(a | (x, y)) r((x, y), a) + \gamma \int_{\mathcal{X}} p(x' | (x, y), a) v((x', y_0)) dx da
\end{aligned}
\end{equation}

Notice that $v_\pi$ is a fixed point of this operator, hence $\tilde{v}_\pi = v_\pi$ and and thus $m_{y_0} \in \M^\infty(\Pi)$ (since our particular choice of $\pi \in \Pi$ was arbitrary). Moreover, we can construct different models for each $y_0 \in \mathcal{Y}$, we know that
\begin{equation}
\label{eq:superfluous_one_to_one}
\M_\mathcal{Y} = \{m_y : y \in \mathcal{Y} \} \subseteq \mathcal{M}^\infty(\mathbbl{\Pi}).
\end{equation}

Moreover, suppose $m_{y_0} \in \M^1(\Pi, \FS)$.
This implies that for all $v \in \FS$
\begin{equation}
\begin{aligned}
&\tilde{\mathcal{T}}_\pi v((x,y)) = \mathcal{T}_\pi v((x, y)) \\
\implies &\int_{\A} \int_{\mathcal{X}} \int_{\mathcal{Y}} \pi(a | (x, y)) p_{m_{y_0}}((x', y') | (x, y), a) v((x', y')) dx' dy' da\\ 
& = \int_\A \int_{\mathcal{X}} \int_{\mathcal{Y}} \pi(a | (x, y)) p((x', y') | (x, y), a) v((x', y')) dx' dy' da \\
\implies &\int_\A \int_{\mathcal{X}} \pi(a | (x, y)) p(x'| (x, y), a) v((x', y_0)) dx' \\
&= \int_\A \int_{\mathcal{X}} \int_{\mathcal{Y}} \pi(a | (x, y))p((x', y') | (x, y), a) v((x', y')) dx' dy'
\end{aligned}
\end{equation}
we now choose $v((x, y)) = \boldsymbol{1}\{ y \neq y_0 \}$ which reduces the above equations to:
\begin{equation}
\label{eq:zero_prob}
\begin{aligned}
\implies &0 = \int_\A \int_{\mathcal{X}} \int_{\mathcal{Y} \neq y_0} \pi(a | (x, y))p((x', y') | (x, y), a) = \mathbb{P}(y' \neq y_0 | x, y, \pi) \\
\implies &\mathbb{P}(y' = y_0 | x, y, \pi) = 1
\end{aligned}
\end{equation}
where $\mathbb{P}$ denotes the conditional probability of an event.

Now consider the class of models defined by Eq.~(\ref{eq:superfluous_one_to_one}). 
Suppose $\M_\mathcal{Y} \in \M^1(\Pi, \FS)$, by Eq.~(\ref{eq:zero_prob}) this would mean that $\mathbb{P}(y' = y_0 | x, y, \pi) = 1$ for all $y_0 \in \mathcal{Y}$. 
This is impossible unless $|\mathcal{Y}| = 1$ hence there must exist $m_{y_0} \notin \M^1(\Pi, \FS)$ and thus $\M^1(\Pi, \FS) \subset \M^\infty(\Pi, \FS)$.

\end{proof}

\FPVEDetOptimal*
\begin{proof}
Denote a deterministic optimal policy with respect to the environment as $\pi^*$. Let $\tilde{m} \in \mathcal{M}^\infty(\mathbbl{\Pi}_\text{det})$ and $\tilde{\pi}^*$ be a deterministic optimal policy with respect $\tilde{m}$.

Suppose $\tilde{\pi}^*$ were not optimal in the environment.
This implies that $v_{\pi^*}(s) \geq v_{\tilde{\pi}^*}(s) \forall s \in \mathcal{S}$ with strict inequality for at least one state.
However, since $\pi^*$ and $\tilde{\pi}^*$ are deterministic
we have:
\begin{equation}
\tilde{v}_{\pi^*}(s) = v_{\pi^*}(s) > v_{\tilde{\pi}^*}(s) = \tilde{v}_{\tilde{\pi}^*}(s)
\end{equation}
for some $s \in \mathcal{S}$. This contradicts $\tilde{\pi}^*$ being optimal in the model.
\end{proof}

\FPVEDetBigger*

\begin{proof}
Since the environment and model only differ when action $R$ is taken from state $2$, we only need to consider deterministic policies that make this choice. Note that if action $R$ is taken from state $2$, the values in both the model and environment at states $2$ and $3$ are necessarily $0$ and the value of each in state $1$ is either $(1 - \gamma)^{-1}$ or $0$ depending on the action taken from state $1$. This suffices to show that the environment and model have the same values for all deterministic policies. 

However, one can see that the model and environment differ for stochastic policies. Take, for instance, a policy for which $\pi(a | s) = 0.5$ for all $a \in \A$, $s \in \S$. The induced Markov reward processes from applying this policy to the environment and model, which share the same reward structure, have different transition dynamics at state $2$. It can be easily verified that this results in different values for the environment and model.
\end{proof}



\PVEBound*
\begin{proof}
We begin by considering the left-hand side
\begin{equation}
\begin{aligned}
\| v_\pi - \tilde{\mathcal{T}}^k_\pi v_\pi \|_{\infty} &= \| v_\pi - \tilde{\mathcal{T}}^k_\pi v + \tilde{\mathcal{T}}^k_\pi v - \tilde{\mathcal{T}}^k_\pi v_\pi  \|_{\infty} \\
&\leq \| v_\pi - \tilde{\mathcal{T}}^k_\pi v \|_{\infty} + \| \tilde{\mathcal{T}}^k_\pi v - \tilde{\mathcal{T}}^k_\pi v_\pi \|_{\infty} \\
&= \| v_\pi - \mathcal{T}^n_\pi v + \mathcal{T}^n_\pi v - \tilde{\mathcal{T}}^k_\pi v \|_{\infty} + \| \tilde{\mathcal{T}}^k_\pi v - \tilde{\mathcal{T}}^k_\pi v_\pi \|_{\infty} \\
&\leq \| v_\pi - \mathcal{T}^n_\pi v \|_{\infty} + \| \mathcal{T}^n_\pi v - \tilde{\mathcal{T}}^k_\pi v \|_{\infty} + \| \tilde{\mathcal{T}}^k_\pi v - \tilde{\mathcal{T}}^k_\pi v_\pi \|_{\infty} \\
&\leq \gamma^n \| v_\pi - v \|_{\infty} + \| \mathcal{T}^n_\pi v - \tilde{\mathcal{T}}^k_\pi v \|_{\infty} + \gamma^k \| v_\pi - v \|_{\infty} \\
&= (\gamma^k + \gamma^n) \| v_\pi - v \|_{\infty} + \| \mathcal{T}^n_\pi v - \tilde{\mathcal{T}}^k_\pi v \|_{\infty}
\end{aligned}
\end{equation}
as needed.
\end{proof}

\begin{proposition}
\label{prop:modified_pve_bound}
For any $\pi \in \PS$, $v \in \FS$ and $k, n \in \mathbb{Z}^+$, assuming $\| v_\pi - v \|_{\infty} < g \cdot \| v_\pi - v \|_{d_\pi}$ for some $g \geq 0$, we have that:
\begin{equation}
\| v_\pi - \tilde{\mathcal{T}}_\pi v_\pi \|_{d_\pi} \leq (g \cdot \gamma^k + \gamma^n) \| v_\pi - v \|_{d_\pi} + \| \mathcal{T}^n_\pi v - \tilde{\mathcal{T}}^k_\pi v \|_{d_\pi}
\end{equation}
\end{proposition}
\begin{proof}
\begin{equation}
\begin{aligned}
\| v_\pi - \tilde{\mathcal{T}}^k_\pi v_\pi \|_{d_\pi} &= \| v_\pi - \tilde{\mathcal{T}}^k_\pi v + \tilde{\mathcal{T}}^k_\pi v - \tilde{\mathcal{T}}^k_\pi v_\pi  \|_{d_\pi} \\
&\leq \| v_\pi - \tilde{\mathcal{T}}^k_\pi v \|_{d_\pi} + \| \tilde{\mathcal{T}}^k_\pi v - \tilde{\mathcal{T}}^k_\pi v_\pi \|_{d_\pi} \\
&= \| v_\pi - \mathcal{T}^n_\pi v + \mathcal{T}^n_\pi v - \tilde{\mathcal{T}}^k_\pi v \|_{d_\pi} + \| \tilde{\mathcal{T}}^k_\pi v - \tilde{\mathcal{T}}^k_\pi v_\pi \|_{d_\pi} \\
&\leq \| v_\pi - \mathcal{T}^n_\pi v \|_{d_\pi} + \| \mathcal{T}^n_\pi v - \tilde{\mathcal{T}}^k_\pi v \|_{d_\pi} + \| \tilde{\mathcal{T}}^k_\pi v - \tilde{\mathcal{T}}^k_\pi v_\pi \|_{d_\pi} \\
&\leq \| v_\pi - \mathcal{T}^n_\pi v \|_{d_\pi} + \| \mathcal{T}^n_\pi v - \tilde{\mathcal{T}}^k_\pi v \|_{d_\pi} + \| \tilde{\mathcal{T}}^k_\pi v - \tilde{\mathcal{T}}^k_\pi v_\pi \|_\infty \\
&\leq \gamma^n \| v_\pi - v \|_{d_\pi} + \| \mathcal{T}^n_\pi v - \tilde{\mathcal{T}}^k_\pi v \|_{d_\pi} + \gamma^k \| v_\pi - v \|_\infty \\
&\leq \gamma^n \| v_\pi - v \|_{d_\pi} + \| \mathcal{T}^n_\pi v - \tilde{\mathcal{T}}^k_\pi v \|_{d_\pi} + g \cdot \gamma^k \| v_\pi - v \|_{d_\pi} \\
&= (g \cdot \gamma^k + \gamma^n) \| v_\pi - v \|_{d_\pi} + \| \mathcal{T}^n_\pi v - \tilde{\mathcal{T}}^k_\pi v \|_{d_\pi}
\end{aligned}
\end{equation}
\end{proof}

\subsection{Experimental details - illustrative experiments}
\label{sec:experimental_details}

\subsubsection{Code}
Code to reproduce our illustrative experiments can be found at \url{https://github.com/chrisgrimm/proper_value_equivalence}.

\subsubsection{Computational resources}
Illustrative experiments were performed on three machines each with $4$ NVIDIA GeForce GTX 1080 Ti graphics cards. 

\begin{wrapfigure}{R}{0.3\textwidth}
\vspace{-40pt}
\centering
\includegraphics[width=\linewidth]{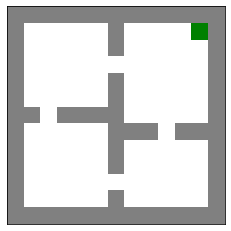}
\caption{Visualization of the Four Rooms environment. }
\label{fig:environments}
\vspace{-40pt}
\end{wrapfigure}

\subsubsection{Environment}
All illustrative experiments depicted in Figures \ref{fig:scatter_plots} and \ref{fig:env_pve_comparison} were carried out in a stochastic version of the Four Rooms environment (depicted in Figure~\ref{fig:environments}) where $|\S| = 104$ and $\A$ consists of four actions corresponding to an intended movement in each of the cardinal directions. When an agent takes an action, it will move in the intended direction 80\% of the time and otherwise move in a random direction. If the agent moves into a wall it will remain in place. When the agent transitions into the upper-right square it receives a reward of $1$, all other transitions yield $0$ reward. 

\subsubsection{Model representation and initialization}
\label{sec:model_rep}

Models are represented tabularly by matrices $\tilde{R} \in \mathbb{R}^{|\S| \times |\A|}$ and $\tilde{P}^a \in \mathbb{R}^{| \S | \times | \S |}$ for $a \in \A$ where $\tilde{R}_{s,a} = \tilde{r}(s,a)$ and $\tilde{P}^a_{s, s'} = \tilde{p}(s' | s, a)$. 
We generally constrain a matrix to be row-stochastic by parameterizing it with an unconstrained matrix of the same shape and applying a softmax with temperature $1$ to each of its rows. 
In experiments with model capacity constraints we additionally impose that each $\tilde{P}^a$ has a rank of at most $k$ by representing $\tilde{P}^a = D^a K^a$ where $D^a \in \mathbb{R}^{|\S| \times k}$, $K^a \in \mathbb{R}^{k \times |\S|}$ and both $D^a$ and $K^a$ are constrained to be row-stochastic (note that the product of row-stochastic matrices is itself row-stochastic). In this setting the parameters of the capacity-constrained transition dynamics are the unconstrained matrices parameterizing $D^a$ and $K^a$.

Models are initialized by randomly sampling the entries of $\tilde{R}$  according to $U(-1, 1)$ and the entries of the matrices parameterizing the transition dynamics according to $U(-5, 5)$, where $U(l, u)$ denotes a uniform distribution over the interval $(l, u)$. 

In all illustrative experiments we train our models using the Adam optimizer with default hyperparameters ($\beta_1 = 0.99$, $\beta_2 = 0.999$, $\epsilon$ =1e-8).  

\subsubsection{Model space experiments}
In Figure~\ref{fig:scatter_plots} we illustrate the properties of spaces of models trained to be in $\M^k(\PS, \FS)$ for $k \in \{1, 30, 40, 50, 60\}$ and in $\M^\infty(\PS)$. To train each of these models we construct a set of policies and functions $\mathcal{D} = \{ (\pi_i, v_i) \}_{i=1}^{100,000}$. Each generated policy $\pi_i$ is, with equal probability, either a uniformly sampled deterministic policy or a stochastic policy for which at each state $s$,  $\pi_i(a | s) = f_a / \sum_{a \in \A} f_a$ where $f_a \sim U(0, 1)$ for each $a \in \A$. Each $v_i$ is sampled such that $v_i(s) \sim U(-1, 1)$ for each $s \in \S$. We then sample minibatches $B \sim \mathcal{D}$ with $|B| = 50$ at each iteration and update models to minimize
\begin{equation}
\label{eq:sampled_losses}
\frac{1}{|B|} \sum_{(\pi, v) \in B} (\tilde{\mathcal{T}}^k_{\pi} v - \mathcal{T}^k_{\pi} v)^2 \hspace{10pt}\text{ and }\hspace{10pt} \frac{1}{|B|} \sum_{(\pi, v) \in B} (\tilde{\mathcal{T}}_\pi v_\pi - v_\pi)^2
\end{equation}
for order-$k$ VE and PVE models respectively. 

Each model is updated in this manner for $500,000$ iterations with a learning rate of 1e-3 and a snapshot of the model is stored every $1000$ iterations---creating a timeline of the model's progress through training. For each model class, this experiment is repeated with $120$ randomly initialized models. 
To generate the points on the scatter plots depicted in Figure~\ref{fig:scatter_plots}, we iterate through the snapshots of these $120$ models. At snapshot $t$ (training iteration $1000 \times t$) we collect the snapshots of all the models and convert each model into a 1D vector representation by concatenating the entries from its reward and transition dynamics matrices. We then apply principle component analysis to these vectors, isolating the first two principle components, which we treat as (x, y) coordinates in the scatter plots. 
For the top row in Figure~\ref{fig:scatter_plots} we color these points according to progress through training: $(t / 500)$. On the bottom row, we compute the optimal policy with respect to each point's corresponding model: $\tilde{\pi}^*$ and color the point according to $(\sum_{s} v_{\tilde{\pi}^*}(s)) / (\sum_{s} v_{\pi*}(s))$.

We produce the plot of model class diameters in Figure~\ref{fig:scatter_plots} by taking the scatter-plot points corresponding to the final snapshot ($t = 500$) of models for each $k$, randomly grouping them into $4$ sets of $30$ points and computing the diameters of each set. We then use these 4 diameters to produce error bars.

\subsubsection{Individual model visualization}
To generate the visualization of the dynamics of individual models displayed in Figure~\ref{fig:env_pve_comparison}b, we randomly select a single PVE model trained in our model space experiments. We then collect $5000$ length $30$ trajectories starting from the bottom left state. The paths of these trajectories are then overlaid on top of a visualization of the environment and colored according to time along the trajectory ($t / 30$). This procedure is repeated using the environment in Figure~\ref{fig:env_pve_comparison}a.

\subsubsection{Model capacity experiment}
We compare the effect of capacity constraints on learning models in $\M^\infty(\PS)$ and $\M^\infty(\PSdet)$ respectively by restricting the rank of the learned model's transition dynamics (as in \ref{sec:model_rep}). We restrict the ranks of model transition dynamics to be at most $k$ for $k \in \{ 20, 30, 40, 50, 60, 70, 80, 90, 100, 104 \}$. 
To train each model we collect a set of $1000$ policies by beginning with a random policy and repeatedly running the policy iteration algorithm in the environment, starting with a randomly initialized policy and stopping when the optimal policy is reached. The sequence of improved policies resulting from this process is stored. Whenever the algorithm terminates, a new random policy is generated and the process is repeated until $1000$ policies have been stored. 
To increase the number of distinct policies generated by this process, at each step of policy iteration, we select, uniformly at random, 10\% of states and update the policy at only these states.
We then further boost the breadth of our collected policies and specialize them to $\PS$ and $\PSdet$ by adding stochastic or deterministic ``noise.''

Precisely, when training a model to be in $\M^\infty(\PS)$ we iterate over each of the 1000 policies generated by our policy iteration procedure and generate an additional 100 policies. Each additional policy is generated by selecting, uniformly at random, 10\% of the original policy's states and replacing the its distribution at these states with a uniform distribution over actions. 

When training a model to be in $\M^\infty(\PSdet)$ the same procedure is repeated but the original policy's distributions, at the selected states, are replaced by randomly generated deterministic distributions.

In either case, this produces $100,000$ policies which are evaluated in the environment. Together this forms a set of policies and corresponding value functions: $\mathcal{D} = \{ (\pi_i, v_i) \}_{i=1}^{100,000}$ which can be used construct mini-batch PVE losses as described in \eqref{eq:sampled_losses}. Models are trained according to these losses for $1,000,000$ iterations with a learning rate of 5e-4. The errorbars around the environment value of the models' optimal policies at the end of training are reported across $10$ seeds. 

\subsection{MuZero experiment}

\paragraph{Atari.}
We follow the Atari configuration used in \citet{schrittwieser2019mastering}, summarised in Table \ref{tab:atari_hyperparams}.

\begin{table}[h]
\caption{Atari hyperparameters.}
\label{tab:atari_hyperparams}
\begin{center}
\begin{small}
\begin{tabular}{ll}
\toprule
\textsc{Parameter} & \textsc{Value} \\
\midrule
  Start no-ops & [0, 30] \\ 
  Terminate on life loss &  Yes \\
  Action set & Valid actions \\
  Max episode length & 30 minutes (108,000 frames) \\
  Observation size & $96\times96$ \\
  Preprocessing & Grayscale \\
  Action repetitions & 4 \\
  Max-pool over last N action repeat frames & 4 \\
  Total environment frames, including skipped frames & 500M \\
\bottomrule
\end{tabular}
\end{small}
\end{center}
\vskip -0.1in
\end{table}  

\paragraph{MuZero implementation.}
Our MuZero implementation largely follows the description given by \citet{schrittwieser2019mastering}, but uses a Sebulba distributed architecture as described in \citet{hessel2021b}, and TD($\lambda$) rather than $n$-step value targets.
The hyperparameters are given in Table \ref{tab:muzero_hyperparams}.
Our network architecture is the same as used in MuZero \citep{schrittwieser2019mastering}.

The base MuZero loss is given by
\begin{equation}
    \mathcal{L}^{\textrm{base}}_t = \mathbb{E}_{\pi} \sum_{k=0}^K\left[ \ell^r(r^{\text{target}}_{t+k}, \hat{r}_t^k) + \ell^v(v_{t+k}^{\text{target}}, \hat{v}_t^k) + \ell^\pi(\pi_{t+k}^{\text{target}}, \hat{\pi}_t^k)\right].
    \label{eq:mz-loss}
\end{equation}

The reward loss $\ell^r$ simply regresses the model-predicted rewards to the rewards seen in the environment.
To compute the value and policy losses, MuZero performs a Monte-Carlo tree search using the learned model.
The policy targets are proportional to the MCTS visitation counts at the root node.
The value targets are computed using the MCTS value prediction $\tilde{v}$ and the sequences of rewards.
MuZero uses an $n$-step bootstrap return estimate $v^{\text{target}}_{t+k} = \sum_{j=1}^{n}\gamma^{j-1}r_{t+k+j} + \gamma^{n}\tilde{v}_{t+k+n}$.
We use the a TD($\lambda$) return estimate instead.

For our additional loss corresponding to past policies, we periodically store the parameters for the value function and policy (i.e. the network heads that take the model-predicted latent state as input).
Then, we compute the same value loss $\ell^v$ for each past value function.
To account for the fact that the reward sequence was drawn from the current policy $\pi$ rather than the stored policies, we use V-trace to compute a return estimate for the past policies.

The additional hyperparameters for the buffer of past value heads were tuned on MsPacman, over a buffer size in $\{64, 128, 256\}$ and an update interval in $\{10, 50, 100, 500\}$.
Our experiments took roughly 35k TPU-v3 device-hours for both tuning and the full evaluation.

\paragraph{Full results.}
We report the final scores per game in Table \ref{tab:atari_full}. The mean scores are across the final 200 episodes in each of three seeds.
We also report the standard error of the mean across seeds only. 
Performing a Wilcoxon signed rank test comparing per-game scores, we find that the version with the additional Past Policies loss has a better final performance with $p=0.044$.

\begin{table}[h]
\caption{Hyperparameters for our MuZero experiment.}
\label{tab:muzero_hyperparams}
\vskip 0.1in
\begin{center}
\begin{small}
\begin{tabular}{ll}
\toprule
\textsc{Hyperparameter} & \textsc{Value} \\
\midrule
  Batch size & 96 sequences \\
  Sequence length & 30 frames \\
  Sequence overlap & 10 frames \\
  Model unroll length $K$ & 5 \\ 
  Optimiser & Adam \\
  Initial learning rate & $1\times10^{-4}$ \\
  Final learning rate (linear schedule) & 0 \\
  Discount & 0.997 \\
  Target network update rate & 0.1 \\
  Value loss weight & 0.25 \\
  Reward loss weight & 1.0 \\
  Policy loss weight & 1.0 \\
  MCTS number of simulations & 25 \\
  $\lambda$ for TD($\lambda$) & 0.8 \\
  MCTS Dirichlet prior fraction & 0.3 \\
  MCTS Dirichlet prior $\alpha$ & 0.25 \\
  Search parameters update rate & 0.1 \\
  Value, reward number of bins & 601 \\
  Nonlinear value transform & $\text{sgn}(z)(\sqrt{|z| +1} - 1) + 0.01z$ \\
\midrule
  Value buffer size & 128 \\
  Value buffer update interval & 50 \\
  Value buffer loss weight & 0.25 \\
\bottomrule
\end{tabular}
\end{small}
\end{center}
\vskip -0.1in
\end{table}

\begin{table}[h]
\centering
\begin{tabular}{l|rl|rl}
\toprule
\multicolumn{1}{l}{Environment} & \multicolumn{2}{c}{MuZero (our impl.)} & \multicolumn{2}{c}{MuZero + Past Policies} \\
\midrule
alien & 38,698 & $\pm$ 2,809 & \textbf{52,821} & $\pm$ 1,918 \\
amidar & \textbf{6,631} & $\pm$ 568 & 4,239 & $\pm$ 1,550 \\
assault & 35,876 & $\pm$ 550 & 35,013 & $\pm$ 738 \\
asterix & \textbf{674,573} & $\pm$ 88,318 & 549,421 & $\pm$ 9,280 \\
asteroids & 214,034 & $\pm$ 4,719 & \textbf{235,543} & $\pm$ 14,605 \\
atlantis & 835,445 & $\pm$ 92,290 & 845,409 & $\pm$ 60,318 \\
bank\_heist & 837 & $\pm$ 265 & 552 & $\pm$ 234 \\
battle\_zone & 39,471 & $\pm$ 12,658 & \textbf{72,183} & $\pm$ 11,385 \\
beam\_rider & 120,675 & $\pm$ 16,588 & 130,129 & $\pm$ 14,014 \\
berzerk & 22,449 & $\pm$ 3,780 & \textbf{35,249} & $\pm$ 3,179 \\
bowling & \textbf{59} & $\pm$ 0 & 47 & $\pm$ 7 \\
boxing & 99 & $\pm$ 0 & 99 & $\pm$ 0 \\
breakout & 504 & $\pm$ 165 & \textbf{770} & $\pm$ 12 \\
centipede & 400,268 & $\pm$ 32,821 & \textbf{534,432} & $\pm$ 38,912 \\
chopper\_command & 524,655 & $\pm$ 154,540 & 660,503 & $\pm$ 27,000 \\
crazy\_climber & 189,621 & $\pm$ 7,313 & \textbf{217,204} & $\pm$ 12,764 \\
defender & 322,472 & $\pm$ 105,043 & \textbf{483,394} & $\pm$ 11,589 \\
demon\_attack & 131,963 & $\pm$ 3,819 & 112,140 & $\pm$ 17,739 \\
double\_dunk & 3 & $\pm$ 4 & -1 & $\pm$ 1 \\
enduro & 0 & $\pm$ 0 & \textbf{132} & $\pm$ 86 \\
fishing\_derby & -97 & $\pm$ 0 & \textbf{-52} & $\pm$ 29 \\
freeway & 0 & $\pm$ 0 & 0 & $\pm$ 0 \\
frostbite & 3,439 & $\pm$ 1,401 & \textbf{8,049} & $\pm$ 526 \\
gopher & \textbf{121,984} & $\pm$ 338 & 120,551 & $\pm$ 923 \\
gravitar & 2,807 & $\pm$ 123 & \textbf{3,927} & $\pm$ 54 \\
hero & 7,877 & $\pm$ 960 & \textbf{9,871} & $\pm$ 523 \\
ice\_hockey & -6 & $\pm$ 4 & -11 & $\pm$ 3 \\
jamesbond & \textbf{23,475} & $\pm$ 1,586 & 13,668 & $\pm$ 4,480 \\
kangaroo & 9,659 & $\pm$ 2,389 & 10,465 & $\pm$ 2,835 \\
krull & 11,259 & $\pm$ 173 & 11,295 & $\pm$ 108 \\
kung\_fu\_master & 55,242 & $\pm$ 4,267 & \textbf{83,705} & $\pm$ 6,565 \\
montezuma\_revenge & 0 & $\pm$ 0 & 0 & $\pm$ 0 \\
ms\_pacman & 40,263 & $\pm$ 387 & \textbf{43,700} & $\pm$ 1,042 \\
name\_this\_game & 76,604 & $\pm$ 7,107 & \textbf{94,974} & $\pm$ 9,942 \\
phoenix & 67,119 & $\pm$ 9,747 & 49,919 & $\pm$ 10,573 \\
pitfall & \textbf{-2} & $\pm$ 1 & -24 & $\pm$ 7 \\
pong & -7 & $\pm$ 9 & -6 & $\pm$ 9 \\
private\_eye & 193 & $\pm$ 101 & -6 & $\pm$ 228 \\
qbert & 64,732 & $\pm$ 8,619 & 70,593 & $\pm$ 16,955 \\
riverraid & 27,688 & $\pm$ 1,001 & 28,026 & $\pm$ 1,823 \\
road\_runner & 151,639 & $\pm$ 90,186 & \textbf{571,829} & $\pm$ 106,184 \\
robotank & \textbf{53} & $\pm$ 2 & 25 & $\pm$ 8 \\
seaquest & 27,530 & $\pm$ 10,632 & \textbf{141,725} & $\pm$ 48,000 \\
skiing & \textbf{-27,968} & $\pm$ 1,346 & -30,062 & $\pm$ 248 \\
solaris & 1,544 & $\pm$ 140 & 1,501 & $\pm$ 193 \\
space\_invaders & 3,962 & $\pm$ 102 & \textbf{5,367} & $\pm$ 953 \\
star\_gunner & 663,896 & $\pm$ 80,698 & 547,226 & $\pm$ 126,538 \\
surround & 7 & $\pm$ 0 & 6 & $\pm$ 1 \\
tennis & -23 & $\pm$ 0 & \textbf{0} & $\pm$ 0 \\
time\_pilot & \textbf{267,331} & $\pm$ 15,256 & 228,282 & $\pm$ 10,844 \\
tutankham & 134 & $\pm$ 10 & 150 & $\pm$ 8 \\
up\_n\_down & 434,746 & $\pm$ 3,905 & 432,240 & $\pm$ 4,221 \\
venture & 0 & $\pm$ 0 & 0 & $\pm$ 0 \\
video\_pinball & 376,660 & $\pm$ 37,647 & 378,897 & $\pm$ 26,486 \\
wizard\_of\_wor & \textbf{79,425} & $\pm$ 1,458 & 54,093 & $\pm$ 6,294 \\
yars\_revenge & 317,803 & $\pm$ 62,785 & 423,271 & $\pm$ 54,094 \\
zaxxon & 15,752 & $\pm$ 231 & 15,790 & $\pm$ 196 \\
\bottomrule
\end{tabular}
\caption{Final Atari scores for our deep RL experiments. We report the mean of the final 200 episodes over all three seeds, and the standard error of the mean across seeds.}
\label{tab:atari_full}
\end{table}
\end{document}